\documentclass[sigconf, nonacm]{acmart}

\usepackage{epsfig}
\usepackage{endnotes}

\usepackage{newtxtext}
\usepackage{newtxmath}

\usepackage{amsthm}
\usepackage{amssymb}

\usepackage{graphicx}
\usepackage{url}            %
\usepackage{xcolor}         

\usepackage[labelformat=simple]{subcaption}
\usepackage{enumerate}
\usepackage{array}
\usepackage{multirow}
\usepackage{caption}
\usepackage{float}
\usepackage{color,xcolor}
\usepackage{afterpage}
\usepackage{siunitx}
\usepackage{microtype}

\usepackage{lineno}
\usepackage{booktabs}

\usepackage[linesnumbered,vlined,ruled]{algorithm2e}

\usepackage{algorithmic}

\usepackage{graphicx}

\usepackage{wasysym}
\usepackage{makecell}

\newcommand\vldbdoi{XX.XX/XXX.XX}
\newcommand\vldbpages{XXX-XXX}
\newcommand\vldbvolume{14}
\newcommand\vldbissue{1}
\newcommand\vldbyear{2026}
\newcommand\vldbauthors{\authors}
\newcommand\vldbtitle{\shorttitle} 
\newcommand\vldbavailabilityurl{URL_TO_YOUR_ARTIFACTS}
\newcommand\vldbpagestyle{plain} 

\begin{document}
\title{Differentially Private Natural Gradient Descent}

\author{Li Pan}
\affiliation{%
  \institution{Institute of Information Engineering, Chinese Academy of Sciences}
  \city{Beijing}
  \country{China}
}
\email{lipan@iie.ac.cn}

\author{Chen Kai}
\affiliation{%
  \institution{Institute of Information Engineering, Chinese Academy of Sciences}
  \city{Beijing}
  \country{China}
}
\email{chenkai@iie.ac.cn}

\author{Chang Shuai}
\affiliation{%
  \institution{Institute of Information Engineering, Chinese Academy of Sciences}
  \city{Beijing}
  \country{China}
}
\email{changshuai751x@iie.ac.cn}

\author{Zhang Sheng Zhi}
\affiliation{%
    \institution{Department of Computer Science, Metropolitan College, Boston University}
  \city{Boston}
  \country{United States}
}
\email{shengzhi@bu.edu}

\author{Lv Pei Zhuo}
\affiliation{%
  \institution{Nanyang Technological University}
  \city{Singapore}
  \country{Singapore}
}
\email{lvpeizhuo@gmail.com}

\author{He Jin Wen}
\affiliation{%
  \institution{Institute of Information Engineering, Chinese Academy of Sciences}
  \city{Beijing}
  \country{China}
}
\email{hejinwen@iie.ac.cn}
\begin{abstract}

Under a fixed privacy budget, the utility of differentially private (DP) training is ultimately determined by its optimization efficiency. Standard first-order DP optimizers such as DP-SGD rely solely on local gradients and ignore the underlying loss curvature. This geometric blindness causes severe zigzagging in ill-conditioned landscapes, squandering precious privacy budgets on inefficient iterations. Practitioners are thus trapped in a bind: either stop training prematurely or inject massive per-step noise, both of which critically compromise final model utility. Natural Gradient Descent (NGD) resolves this by preconditioning gradients with curvature, aligning updates with the loss geometry and extracting more efficient signal from every noisy step, offering a principled pathway to break the privacy-utility bottleneck.

Despite its theoretical appeal, directly integrating NGD with DP introduces fundamental challenges: curvature estimation itself consumes prohibitive privacy budgets, isotropic DP operations conflict with the anisotropic scaling of NGD, and the inverse curvature catastrophically amplify parameter updates in flat directions, causing training instability. We propose DP-NGD, a practical framework that systematically addresses these obstacles by decoupling curvature estimation from private data, reconciling isotropic DP constraints with anisotropic second-order optimization via a whitened-space mechanism, and dynamically clamping the curvature to stabilize training. Extensive experiments on standard benchmarks demonstrate that DP-NGD achieves state-of-the-art accuracy, breaking through the utility ceilings of first-order baselines while delivering up to a $10\times$ convergence speedup under the same privacy budget.

\end{abstract}

\maketitle

\pagestyle{\vldbpagestyle}
\begingroup\small\noindent\raggedright\textbf{PVLDB Reference Format:}\\
\vldbauthors. \vldbtitle. PVLDB, \vldbvolume(\vldbissue): \vldbpages, \vldbyear.\\
\href{https://doi.org/\vldbdoi}{doi:\vldbdoi}
\endgroup
\begingroup
\renewcommand\thefootnote{}\footnote{\noindent
This work is licensed under the Creative Commons BY-NC-ND 4.0 International License. Visit \url{https://creativecommons.org/licenses/by-nc-nd/4.0/} to view a copy of this license. For any use beyond those covered by this license, obtain permission by emailing \href{mailto:info@vldb.org}{info@vldb.org}. Copyright is held by the owner/author(s). Publication rights licensed to the VLDB Endowment. \\
\raggedright Proceedings of the VLDB Endowment, Vol. \vldbvolume, No. \vldbissue\ %
ISSN 2150-8097. \\
\href{https://doi.org/\vldbdoi}{doi:\vldbdoi} \\
}\addtocounter{footnote}{-1}\endgroup

\ifdefempty{\vldbavailabilityurl}{}{
\vspace{.3cm}
\begingroup\small\noindent\raggedright\textbf{PVLDB Artifact Availability:}\\
The source code, data, and/or other artifacts have been made available at 
\url{https://github.com/XXX}.
\endgroup
}

\section{Introduction}

Differential Privacy (DP)~\cite{dp1,dp2} has emerged as the gold standard for privacy-preserving machine learning. Under a fixed privacy budget, the final utility of DP training is bottlenecked by its optimization efficiency. Standard first-order DP optimizers, such as DP-SGD~\cite{dpsgd1,dpsgd2} and its adaptive variants, suffer from inherent ``geometric blindness'': they rely solely on local gradients and ignore the loss curvature. This leads to the severe zigzagging and slow convergence in highly ill-conditioned loss landscapes typical of deep networks, as illustrated in Fig.~\ref{fig:motivation_sgd}. In non-private training, slow convergence is often harmless, since one can simply run more iterations. However, in DP training, every iteration strictly consumes the privacy budget. This inefficiency squanders the precious budget on uninformative steps and traps practitioners in a bind: either stop training prematurely or inject massive per-step noise, both of which critically compromise final model utility.

Natural Gradient Descent (NGD)~\cite{natural_gradient,natural_gradient2}, a classic second-order optimization method, fundamentally resolves this geometric blindness. As depicted in Fig.~\ref{fig:motivation_ngd}, by preconditioning the gradient with a curvature matrix $F$ (e.g., the Fisher Information Matrix), NGD corrects the update direction to point more directly toward the local optimum, significantly accelerating convergence across ill-conditioned landscapes. Crucially, this acceleration directly translates to higher utility under the same privacy guarantee. Under standard DP accounting mechanisms (e.g., Rényi DP~\cite{rdp1,rdp2} or Moments Accountant~\cite{moments_account1,moments_account2}), the required per-step noise multiplier $\sigma$ scales as $\sigma \propto \sqrt{T}$ for a fixed total privacy budget. By drastically reducing the required number of iterations $T$, NGD enables substantially lower per-step noise injection, thereby elevating the signal-to-noise ratio (SNR) throughout training. This provides a mathematically principled pathway to break the longstanding privacy-utility bottleneck.

\begin{figure*}[htbp] 
    \centering
    
    \begin{subfigure}[b]{0.48\textwidth}
        \centering
        
        \includegraphics[width=\textwidth]{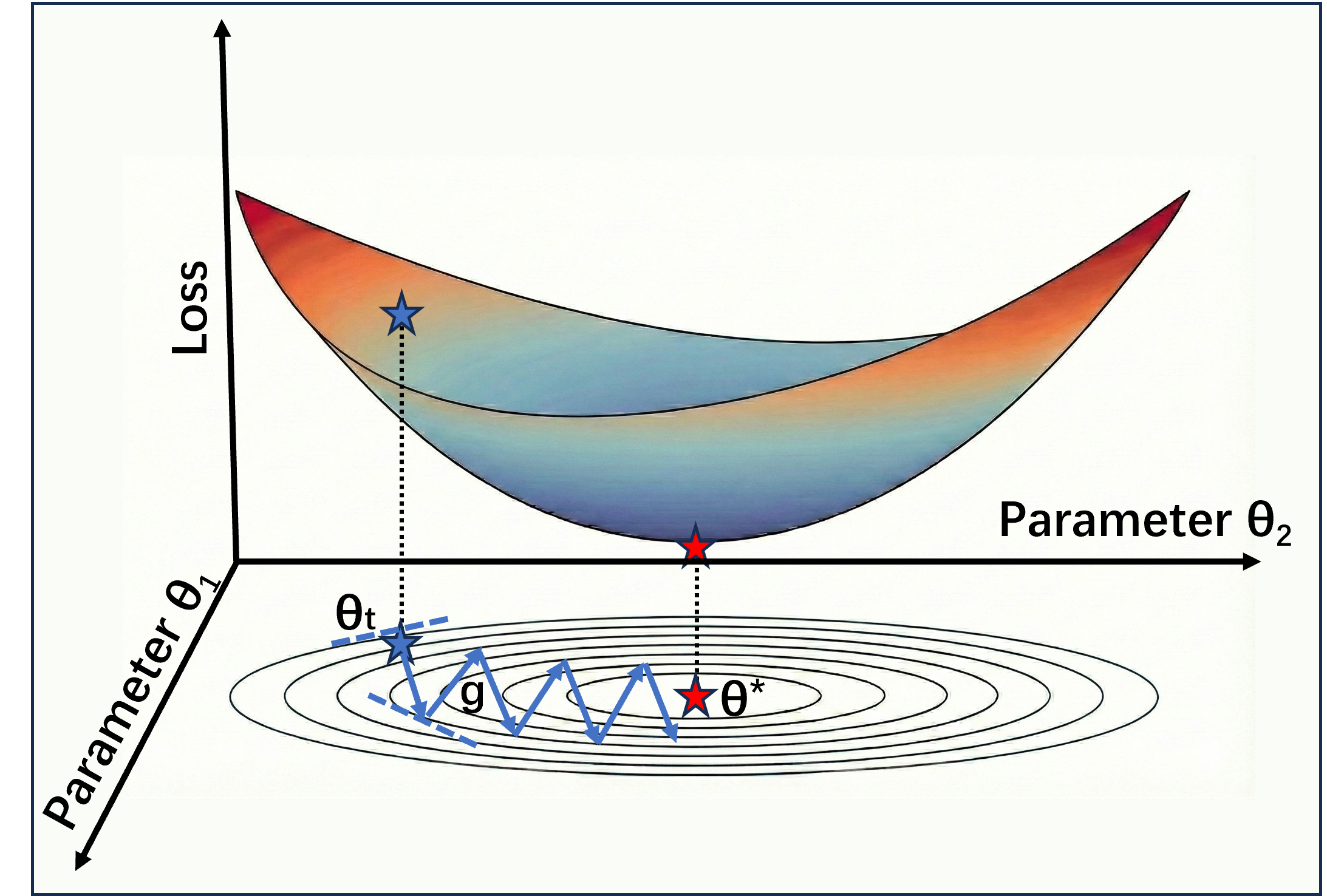}
        \caption{SGD (geometry-blind)}
        \label{fig:motivation_sgd}
    \end{subfigure}
    \hfill 
    \begin{subfigure}[b]{0.48\textwidth}
        \centering
        
        \includegraphics[width=\textwidth]{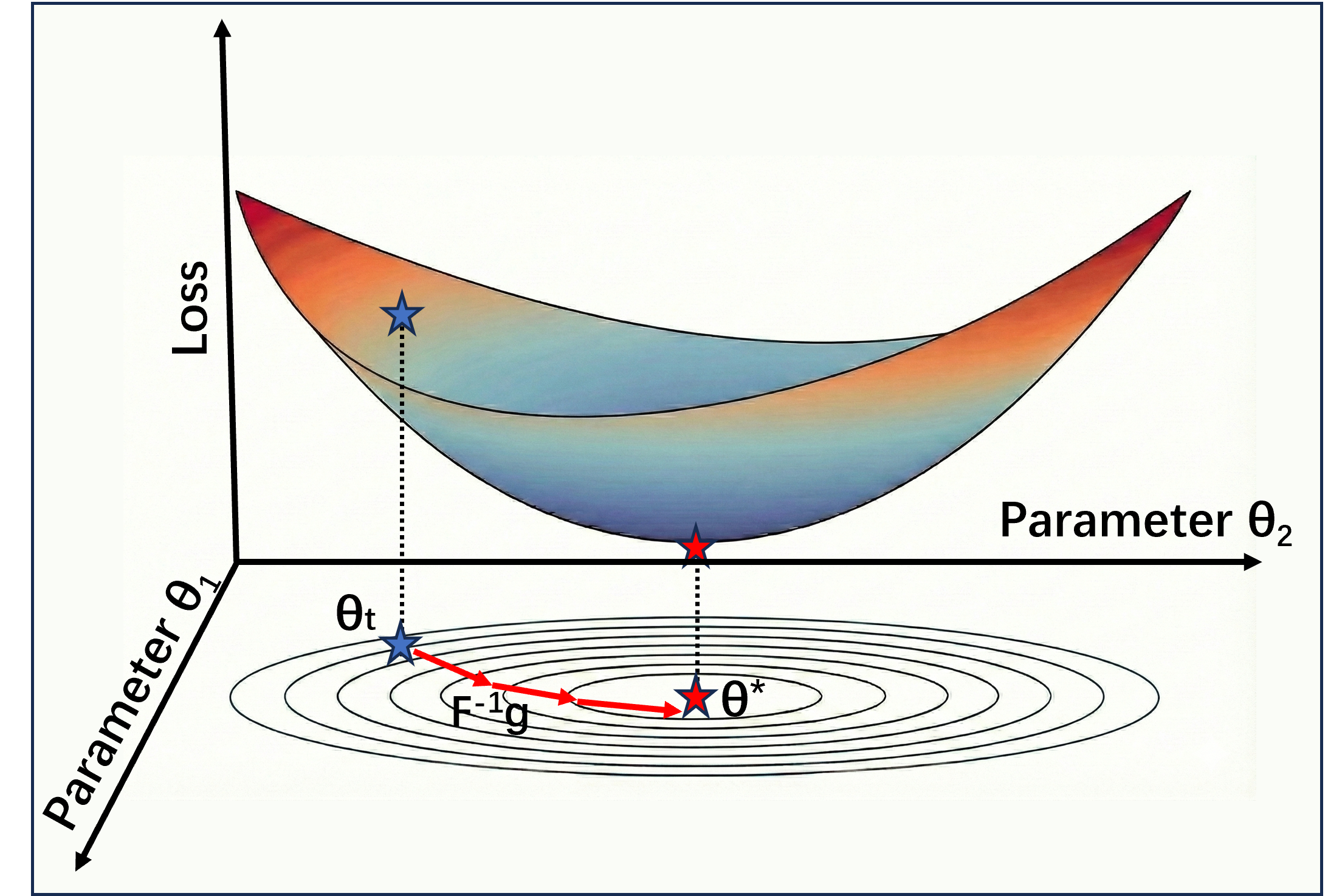}
        \caption{NGD (curvature-aware)}
        \label{fig:motivation_ngd}
    \end{subfigure}
    
    \vspace{-5pt} 
    
    \caption{Why optimization efficiency dictates differential privacy utility. Here, $\theta_t$ represents the current model parameters and $\theta^*$ is the local optimum. \textbf{(a) SGD (geometry-blind).} The gradient $g$ simply follows the orthogonal direction of the loss contours rather than pointing toward the local optimum. This blindness induces severe zigzagging and slow convergence, rapidly depleting the fixed privacy budget. \textbf{(b) NGD (curvature-aware).} Second-order methods precondition the gradient with the inverse curvature $F^{-1}$. The resulting natural gradient $F^{-1}g$ points more directly toward $\theta^*$, significantly accelerating convergence and drastically reducing inefficient steps. Because fewer iterations allow the fixed budget to be concentrated, it significantly lowers the per-step noise ($\sigma \propto \sqrt{T}$) and raises the effective signal-to-noise ratio, ultimately unlocking superior model utility.}
    \label{fig:motivation}
\end{figure*}

Despite its theoretical appeal, integrating NGD with DP is hindered by three fundamental challenges. First, standard NGD requires accurate curvature estimation from private data, which would consume a prohibitive fraction of the privacy budget in DP settings. Diverting the budget $\epsilon$ to this auxiliary computation starves the primary optimization process, largely negating any theoretical SNR gain. Second, differential privacy mandates isotropic clipping and noise injection to uniformly bound sensitivity, whereas NGD relies on anisotropic, curvature-aware scaling to achieve acceleration. Directly imposing isotropic DP constraints on natural gradients corrupts the geometric properties essential for second-order optimization. Third, in flat directions of the loss landscape where the eigenvalues of $F$ approach zero ($\lambda_i \to 0$), the inverse curvature $F^{-1}$ magnifies the noisy gradients by $1/\lambda_i \to \infty$, triggering training instability and catastrophic divergence.

To systematically resolve these bottlenecks, we propose DP-NGD, a practical and robust framework. First, we eliminate the privacy overhead of curvature estimation by computing $F$ solely on a small public auxiliary dataset to extract structural priors, thereby preserving the entire privacy budget for gradient updates. Second, we resolve the geometric incompatibility via an $F^{-1/2}$-whitened space update mechanism. We prove that performing isotropic DP operations in this whitened space mathematically translates to applying anisotropic, curvature-aware updates in the original parameter space. Crucially, this theoretical alignment reveals a profound connection---it unifies the DP sensitivity constraint with the KL trust region of NGD. Finally, to prevent parameter explosion in flat directions, we derive a dynamic clamping floor from the expected Euclidean step size of DP-SGD. By clamping the curvature eigenvalues consistently above this threshold, we effectively suppress flat-direction parameter explosion while preserving second-order acceleration in steep directions.

\vspace{5pt}
\noindent \textbf{Contributions.} Our main contributions are summarized as follows:

\vspace{2pt}
\noindent $\bullet$ \textbf{Privacy-Free Curvature Acceleration.} Our key insight is that DP noise disproportionately corrupts fine-grained gradient components, making DP training inherently tolerant of curvature estimation errors in these directions. Consequently, coarse-grained structural priors suffice for effective second-order acceleration. We realize this through a decoupled DP-NGD framework that estimates the curvature entirely on a small public auxiliary dataset, incurring zero additional privacy cost over standard DP-SGD.

\vspace{2pt}
\noindent $\bullet$ \textbf{KL-DP Duality.} We prove that applying isotropic DP operations within an $F^{-1/2}$-whitened space mathematically translates to curvature-aligned anisotropic updates in the original parameter space. This establishes a rigorous KL-DP Duality that unifies the DP sensitivity constraint with the KL trust region of NGD. To ensure stable training, we further introduce a dynamic clamping mechanism that effectively prevents parameter explosion in flat directions.

\vspace{2pt}
\noindent $\bullet$ \textbf{Superior Efficiency and Utility.} Across multiple benchmarks, DP-NGD consistently achieves state-of-the-art accuracy while dramatically accelerating convergence. For instance, it breaks through the utility ceilings of baselines with up to a $10\times$ convergence speedup (e.g., reaching the peak accuracy of DP-SGD using only $\text{42.8}\%$ of the iterations on CIFAR-10 under $\epsilon=\text{1.0}$). DP-NGD offers a highly practical and robust second-order accelerator for privacy-preserving deep learning.

\section{Preliminaries}
\label{sec: preliminary}

\subsection{Differential Privacy \& DP-SGD}

\textit{Differential Privacy.} Differential Privacy (DP)~\cite{Dong_dp} provides a rigorous mathematical framework for privacy-preserving computation. A randomized algorithm $\mathcal{M}$ satisfies $(\epsilon, \delta)$-DP if, for any two adjacent datasets $D$ and $D'$ differing by at most one example, and for any subset of outputs $S \subseteq \text{Range}(\mathcal{M})$, it holds that:
\begin{equation}
    \Pr[\mathcal{M}(D) \in S] \le e^\epsilon \Pr[\mathcal{M}(D') \in S] + \delta.
\end{equation}

\textit{DP-SGD \& Isotropic Operations.} The standard approach to training deep networks under DP is DP-SGD. At each iteration, it computes the per-sample gradient $g_i$ and clips it to a maximum $L_2$-norm bound $C$, i.e., $\bar{g}_i = g_i \cdot \min(1, \frac{C}{\|g_i\|_2})$. Then, it aggregates the clipped gradients and injects Gaussian noise to ensure privacy:
\begin{equation}
    \tilde{g} = \frac{1}{|\mathcal{B}|} \left( \sum_{i \in \mathcal{B}} \bar{g}_i + \mathcal{N}(0, \sigma^2 C^2 \mathbf{I}) \right),
\end{equation}
where $\mathcal{B}$ is the mini-batch and $\sigma$ is the noise multiplier. Crucially, both the $L_2$-norm clipping boundary and the spherical Gaussian noise covariance ($\sigma^2 C^2 \mathbf{I}$) are strictly isotropic across the entire parameter space.

\textit{Privacy Accounting.} To rigorously track the privacy loss over $T$ iterations, modern implementations typically rely on advanced privacy accounting methods\cite{privacy_account1}, such as the Moments Accountant~\cite{moments_account1} or Rényi Differential Privacy (RDP)~\cite{rdp1,rdp2}. Under standard composition theorems, for a fixed overall privacy budget $(\epsilon, \delta)$ and sampling rate $q = |\mathcal{B}|/N$, the required per-step noise multiplier $\sigma$ asymptotically scales as:
\begin{equation}
   \sigma \propto \frac{q \sqrt{T \log(1/\delta)}}{\epsilon}.
\end{equation}
This establishes a fundamental property: when other privacy parameters are fixed, the injected noise scales with the square root of the total iterations ($\sigma \propto \sqrt{T}$). Consequently, accelerating convergence to reduce $T$ provides a mathematically principled pathway to decrease per-step noise and improve model utility.

\subsection{Natural Gradient Descent \& K-FAC}

\textit{KL Trust Region.} Natural Gradient Descent (NGD) approaches optimization from an information-geometric perspective. Instead of updating parameters in the Euclidean space, NGD minimizes the objective $\mathcal{L}$ within a local KL-divergence trust region~\cite{kl_trust_region1,kl_trust_region2} to maintain distributional stability. The KL-divergence between the predictive distributions of the current parameters $\theta$ and the updated parameters $\theta + \Delta \theta$ can be approximated via its second-order Taylor expansion:
\begin{equation}
    \text{KL}(P_{\theta} \| P_{\theta + \Delta \theta}) \approx \frac{1}{2} \Delta \theta^T F \Delta \theta,
\end{equation}
where $F = \mathbb{E}_{x, y \sim P_{\theta}}[\nabla_{\theta} \log P_{\theta}(y|x) \nabla_{\theta} \log P_{\theta}(y|x)^T]$ is the Fisher Information Matrix (FIM), which captures the local loss curvature. Solving the trust-region optimization yields the NGD update rule:
\begin{equation}
    \Delta \theta = - \eta F^{-1} \nabla_{\theta} \mathcal{L},
\end{equation}
where $\eta$ is the learning rate and $g \triangleq \nabla_{\theta} \mathcal{L}$ denotes the first-order gradient of the loss. In this context, $F^{-1} g$ is formally defined as the natural gradient. Unlike standard first-order methods, preconditioning the gradients with inverse curvature $F^{-1}$ enforces an anisotropic scaling: it aggressively accelerates progress along flat directions while damping updates in steep directions.

\textit{K-FAC Approximation.} In deep neural networks, computing and inverting the exact FIM~\cite{fim1,fim2} is computationally intractable due to the massive parameter dimensionality. The Kronecker-Factored Approximate Curvature (K-FAC)~\cite{kfac2,kfac3} efficiently approximates $F$ by assuming independence between the activations $a_l$ and the pre-activation gradients $g_l$ of each layer $l$. This decouples the layer-wise FIM block $F_l$ into a Kronecker product of two smaller matrices: $F_l \approx A_l \otimes G_l$, where $A_l = \mathbb{E}[a_l a_l^T]$ and $G_l = \mathbb{E}[g_l g_l^T]$. This structural factorization compresses the curvature degrees of freedom from $\mathcal{O}(d^2)$ to $\mathcal{O}(d_{\text{in}}^2 + d_{\text{out}}^2)$, where $d_{\text{in}}, d_{\text{out}} \ll d$, drastically reducing the requisite sample complexity. This provides a mathematical justification for why a limited public dataset suffices to extract reliable coarse-grained geometric priors.

\section{The DP-NGD Framework}
\label{sec:method}

To address the three fundamental challenges that hinder the integration of NGD and DP, we propose DP-NGD, a practical and robust framework built upon three mutually reinforcing components. We introduce each component in the following sections, and finally unify them into a complete training algorithm.

\subsection{Privacy-Free Coarse-Grained Curvature Estimation}
\label{subsec:privacy_free_curvature}

Standard natural gradient descent (NGD) estimates the curvature matrix directly from the training data. In a DP setting, doing so would consume a substantial portion of the privacy budget merely for preconditioning, leaving little budget for the actual gradient updates that drive learning. We avoid this prohibitive overhead through strict computational decoupling: curvature estimation is performed exclusively on a public auxiliary dataset ($\mathcal{D}_{\text{pub}}$), while gradient clipping and DP noise injection are confined entirely to the private training set ($\mathcal{D}_{\text{priv}}$). Because $\mathcal{D}_{\text{pub}}$ is \textit{a priori} public and disjoint from $\mathcal{D}_{\text{priv}}$—lying strictly outside the adjacency relation defining DP—this computation incurs zero privacy cost. Consequently, the entire privacy budget $\epsilon$ is preserved for gradient updates.

A natural question arises: why does an out-of-distribution (OOD) public dataset suffice for effective curvature estimation? We identify a fundamental synergy between DP and NGD, termed \textit{Coarse-Grained Curvature Transferability}. The cornerstone of this synergy is that DP noise disproportionately corrupts fine-grained gradient components. Specifically, in the flat directions of the loss landscape, the true gradient signals are typically small and are therefore dominated by the injected DP noise. Preconditioning these fine-grained directions with inaccurate curvature merely maps random noise into random noise, incurring minimal utility loss. Consequently, DP training is inherently tolerant of curvature estimation errors in these fine-grained dimensions: it demands structural fidelity exclusively in the dominant, coarse-grained eigenspaces to achieve effective second-order acceleration.

Fortunately, this requirement for coarse-grained fidelity is naturally satisfied in our optimization paradigm. First, a shared model architecture imposes strong inductive priors~\cite{inductive_prior1,inductive_prior2}, yielding covariance matrices that are structurally homomorphic across datasets. Second, natural images share universal low-level statistics~\cite{low_level_statistic1,low_level_statistic2}—such as edges, textures, and color smoothness—that dominate the top eigenspaces of the curvature matrix, allowing even OOD data to reliably capture the macroscopically steep directions. Third, the K-FAC approximation acts as a low-pass filter: its expectation over samples effectively averages out fine-grained fluctuations, isolating precisely the statistically stable, coarse-grained structures we rely on. Together, these properties ensure that the coarse-grained curvature priors essential for acceleration can be effectively extracted even from OOD public auxiliary dataset~\cite{dp_public_data1,dp_public_data2}.

Specifically, we precompute the curvature matrix $F$ on $\mathcal{D}_{\text{pub}}$ via the K-FAC approximation. For each layer $l$, the diagonal block $F_l$ is constructed as the Kronecker product of the activation covariance $A_l$ and the pre-activation gradient covariance $G_l$:
\begin{equation}
    F_l = A_l \otimes G_l, \quad A_l = \mathbb{E}_{\mathcal{D}_{\text{pub}}}[a_l a_l^\top], \quad G_l = \mathbb{E}_{\mathcal{D}_{\text{pub}}}[g_l g_l^\top].
\end{equation}
The complete curvature is then assembled as a block-diagonal matrix across all layers: $F = \operatorname{diag}(F_1, F_2, \ldots, F_L)$. Notably, our method requires only a small public auxiliary set for effective curvature estimation (e.g., $\sim 1\%$ of the private training data size; detailed in Section~\ref{sec:sample efficiency}). This extreme sample efficiency stems from two factors. First, K-FAC drastically reduces the requisite sample complexity from $\mathcal{O}(d)$ to the order of a single layer's width. Second, as analyzed above, our framework primarily relies on the dominant eigenspaces and naturally tolerates fine-grained estimation inaccuracies, further relaxing the sample requirement.

\subsection{Reconciling Isotropic Privacy with Anisotropic Optimization}
\label{sec:whitened space update}

Directly imposing isotropic DP constraints on anisotropic natural gradients corrupts the geometric properties essential for second-order optimization. To resolve this conflict, we introduce a Whitened-Space Update Mechanism. By shifting all DP operations into an $F^{-1/2}$-whitened space, we prove that the isotropic DP mechanism translates into anisotropic, curvature-aligned parameter updates in the original space---thus fully preserving the acceleration of NGD under DP constraints. Crucially, the whitening matrix $F^{-1/2}$ is not chosen arbitrarily: it emerges directly from aligning the DP $L_2$-norm constraint with the KL trust region, a derivation we provide in Appendix~\ref{app:whitening_derivation} for completeness.

Algorithm~\ref{alg:whitened_update} formalizes the Whitened-Space Update Mechanism. For each sample in the Poisson-sampled batch, the raw gradient is first projected into the whitened space via $F^{-1/2}$ (Line~2). This geometric transformation creates an isotropic space with respect to the geometry induced by $F$. Within this whitened space, standard $L_2$-norm clipping can be safely applied without distorting the relative scale of curvature directions (Line~3). Following clipping, the gradients are aggregated and perturbed with isotropic Gaussian noise (Lines~5--6). This procedure is identical to standard DP-SGD, except that it operates on the whitened gradients $\hat{g}_i$ rather than the raw gradients $g_i$. Finally, the sanitized gradient is projected back to the original parameter space (Line~7). Since the whitening matrix $F^{-1/2}$ is precomputed entirely on public data, it acts as a linear transformation strictly independent of the sensitive data. By the post-processing property of DP~\cite{post-processing1,post-processing2}, this inverse projection does not incur any additional privacy cost. Consequently, the privacy accounting remains identical to DP-SGD. A formal privacy analysis is provided in Section~\ref{sec:privacy analysis}.

Having established the privacy guarantees of the whitened-space mechanism, we now formally prove that the resulting parameter update is anisotropic and curvature-aligned, successfully recovering the natural gradient in expectation.

\begin{algorithm}[tb]
\caption{Whitened-Space Update Mechanism}
\label{alg:whitened_update}
\begin{algorithmic}[1]
\REQUIRE Poisson-sampled batch $\mathcal{B}$, per-sample gradients $\{g_i\}_{x_i \in \mathcal{B}}$, whitening matrix $F^{-1/2}$, clipping threshold $C$, noise multiplier $\sigma$, expected batchsize $B$, learning rate $\eta$.
\ENSURE Parameter update  $\Delta \theta$
\FOR{each sample $x_i \in \mathcal{B}$}
    \STATE $\hat{g}_i \leftarrow F^{-1/2} g_i$ \COMMENT{Project to the whitened space}
    \STATE $\bar{g}_i \leftarrow \hat{g}_i \cdot \min\left(1, \frac{C}{\|\hat{g}_i\|_2}\right)$ \COMMENT{Isotropic $L_2$ clipping}
\ENDFOR
\STATE $\xi \sim \mathcal{N}\left(0, \sigma^2 C^2 \mathbf{I}\right)$ \COMMENT{Sample isotropic Gaussian noise}
\STATE $\tilde{g} \leftarrow \frac{1}{B} \big( \sum_{i \in \mathcal{B}} \bar{g}_i + \xi \big) $ \COMMENT{Aggregate and inject noise}
\STATE $\Delta \theta \leftarrow -\eta F^{-1/2} \tilde{g}$ \COMMENT{Project back to original space}
\RETURN $\Delta \theta$
\end{algorithmic}
\end{algorithm}

\begin{theorem}[Curvature Alignment of Whitened-Space Updates]
\label{theorem: curvature_alignment}
Let $\Delta \theta$ be the parameter update generated by Algorithm~\ref{alg:whitened_update}, and let $F$ be the estimated curvature matrix. For any sample $i$, define its adaptive clipping scalar as $c_i = \min\left(1, \frac{C}{\|F^{-1/2} g_i\|_2}\right)$. In expectation over the isotropic Gaussian noise $\xi \sim \mathcal{N}(0, \sigma^2 C^2 \mathbf{I})$, the parameter update recovers the adaptively clipped natural gradient:
$$ \mathbb{E}_{\xi}[\Delta \theta] = -\frac{\eta}{B} \sum_{i \in \mathcal{B}} c_i F^{-1} g_i $$
Furthermore, the injected DP noise is transformed into an anisotropic, curvature-aligned distribution in the original parameter space, with a covariance matrix proportional to $F^{-1}$:
$$ \text{Cov}(\Delta \theta) = \frac{\eta^2 \sigma^2 C^2}{B^2} F^{-1} $$
\end{theorem}

\begin{proof}
Based on Algorithm~\ref{alg:whitened_update}, the final parameter update $\Delta \theta$ is a linear combination of the gradient signal and the injected noise. Its complete algebraic expression is formulated as:
\begin{equation}
\label{eq:full_update}
\Delta \theta = -\eta F^{-1/2} \tilde{g} = -\frac{\eta}{B} \sum_{i \in \mathcal{B}} F^{-1/2} \bar{g}_i - \frac{\eta}{B} F^{-1/2} \xi
\end{equation}
\noindent where $\bar{g}_i = \hat{g}_i \min(1, \frac{C}{\|\hat{g}_i\|_2})$ is the clipped gradient in the whitened space, and $\xi \sim \mathcal{N}(0, \sigma^2 C^2 \mathbf{I})$ is the isotropic Gaussian noise. We establish the curvature alignment of the final parameter update by analyzing its expectation and covariance, respectively.

\paragraph{Part I: Parameter Update Expectation.}
For each sample $i$, we denote its adaptive clipping factor as $c_i = \min\left(1, \frac{C}{\|\hat{g}_i\|_2}\right)$. Since $c_i$ is a scalar, it commutes with matrix multiplication. Together with the fact that the injected DP noise is zero-mean ($\mathbb{E}_{\xi}[\xi] = 0$), the stochastic noise term in Eq.~\eqref{eq:full_update} vanishes. Taking the expectation over $\xi$, we obtain:
\begin{equation}
\mathbb{E}_{\xi}[\Delta \theta] = -\frac{\eta}{B} \sum_{i \in \mathcal{B}} F^{-1/2} \left( c_i F^{-1/2} g_i \right) = -\frac{\eta}{B} \sum_{i \in \mathcal{B}} c_i F^{-1} g_i
\end{equation}
Because the clipping factor $c_i$ exclusively rescales the magnitude without altering the vector's orientation, this expectation recovers the adaptively clipped natural gradient $F^{-1}g_i$.

\paragraph{Part II: Parameter Update Covariance.}
Next, we analyze the covariance matrix of the update $\Delta \theta$. Since the first term in Eq.~\eqref{eq:full_update} is deterministic with respect to the DP mechanism, it contributes nothing to the variance. The covariance is entirely determined by the linear transformation of the Gaussian noise term $\xi$. Using the property of covariance under affine transformations, $\text{Cov}(AX) = A \text{Cov}(X) A^T$, we derive:
\begin{align}
\text{Cov}(\Delta \theta) &= \text{Cov}\left( -\frac{\eta}{B} F^{-1/2} \xi \right) \nonumber \\
&= \left(-\frac{\eta}{B} F^{-1/2}\right) \text{Cov}(\xi) \left(-\frac{\eta}{B} F^{-1/2}\right)^T \nonumber \\
&= \frac{\eta^2}{B^2} F^{-1/2} (\sigma^2 C^2 \mathbf{I}) F^{-1/2} \nonumber \\
&= \frac{\eta^2 \sigma^2 C^2}{B^2} F^{-1}
\end{align}
Here we utilize the symmetry of $F^{-1/2}$, a property intrinsically inherited from the Fisher information matrix. This demonstrates that the originally isotropic noise is transformed into an anisotropic distribution whose covariance is precisely shaped by the inverse curvature $F^{-1}$, thereby aligning with the natural gradient.

Combining Part I and Part II, the whitened-Space mechanism preserves the curvature-aligned geometry for both the gradient signal and the DP perturbation, thereby completing the proof.
\end{proof}

Together with the privacy guarantee established, Theorem~\ref{theorem: curvature_alignment} confirms that the Whitened-Space Update Mechanism resolves the inherent geometric conflict between DP and NGD. By achieving this theoretical alignment, our method strictly satisfies the differential privacy constraints while preserving the crucial acceleration properties of second-order optimization.

\vspace{3pt}
\noindent \textbf{The KL-DP Duality.} 
Theorem~\ref{theorem: curvature_alignment} reveals a deeper structural connection within our method: the DP sensitivity bound $C$ simultaneously defines the radius of the KL trust region for NGD.

To avoid destabilizing updates caused by pathological curvature, NGD controls the per-step distributional shift by bounding the KL divergence. Under the standard small-step assumption of NGD, higher-order terms are negligible and the KL divergence is locally governed by
\begin{equation}
    D_{\text{KL}}(P_\theta \parallel P_{\theta+\Delta\theta}) = \frac{1}{2} \Delta\theta^\top F \Delta\theta.
\end{equation}
Consider the parameter update contributed by a single clipped sample in our whitened-space mechanism: $\Delta\theta_i = -\eta F^{-1/2} \bar{g}_i$. Substituting this into the above KL divergence yields:
\begin{align}
D_{\text{KL}} &= \frac{1}{2} \left( -\eta F^{-1/2} \bar{g}_i \right)^\top F \left( -\eta F^{-1/2} \bar{g}_i \right) \nonumber \\
&= \frac{\eta^2}{2} \bar{g}_i^\top \left( F^{-1/2} F F^{-1/2} \right) \bar{g}_i \nonumber \\
&= \frac{\eta^2}{2} \|\bar{g}_i\|_2^2.
\end{align}
Since DP clipping in the whitened space enforces $\|\bar{g}_i\|_2 \le C$, we immediately obtain
\begin{equation}
    \|\bar{g}_i\|_2 \le C \quad \Longrightarrow \quad D_{\text{KL}} \le \frac{\eta^2}{2} C^2.
\end{equation}
This implies that the scalar $C$ plays a dual role. As a privacy safeguard, it bounds the maximal contribution of any individual sample, rigorously ensuring differential privacy. As an optimization boundary, it explicitly limits the distributional shift that any single data point can induce, matching the trust-region constraint required for stable NGD convergence. This marks a fundamental departure from standard DP-SGD, where the clipping threshold is an arbitrary, isotropic bound imposed without regard to the geometry of the loss landscape. In our framework, the DP sensitivity constraint is instead defined in a principled, curvature-aware manner derived directly from second-order optimization theory.

The underlying mechanism behind this duality is the geometric equivalence established by the whitening transformation:
\begin{equation}
    \|\hat{g}_i\|_2 = \|F^{-1/2} g_i\|_2 = \sqrt{g_i^\top F^{-1} g_i} = \|g_i\|_{F^{-1}}.
\end{equation}
Geometrically, this equivalence explains how our mechanism natively adapts to the local loss landscape. Bounding the $L_2$-norm in the whitened space ($\|\hat{g}_i\|_2 \le C$) is mathematically equivalent to bounding the curvature-aware Mahalanobis norm in the parameter space ($\|g_i\|_{F^{-1}} \le C$). Consequently, steep directions (associated with large eigenvalues of $F$) are inherently suppressed to prevent divergence, while informative components along flat directions (small eigenvalues) are faithfully preserved.

\subsection{Dynamic Curvature Clamping via Euclidean Step Equivalence}

Although DP clipping bounds the gradient's $L_2$-norm in the whitened space, the inverse projection via $F^{-1/2}$ leaves the Euclidean step size unbounded in the original parameter space. In flat directions where the curvature eigenvalues approach zero ($\lambda_i \to 0$), parameter updates are magnified by a factor of $1/\sqrt{\lambda_i} \to \infty$, causing severe training instability. Crucially, this parameter explosion can occur without violating the KL trust region: flat directions permit large parameter moves with negligible distributional shift~\cite{flat_direction1,flat_direction2}, making the KL constraint structurally blind to such instability. An explicit Euclidean safeguard on the parameter update is therefore necessary.

\paragraph{Deriving the Euclidean Safety Bound.}
To establish this safeguard, we adopt standard DP-SGD as a trusted baseline. DP-SGD maintains well-controlled expected Euclidean displacements throughout training. By enforcing that the expected displacement of our DP-NGD update does not exceed that provable upper bound, we obtain a rigorous safety floor for robust convergence. 

In standard DP-SGD, the parameter update is formulated as $\Delta \theta_{\text{sgd}} = -\eta_{\text{sgd}}\big(\bar{g}_{\text{sgd}} + \frac{1}{B}\xi_{\text{sgd}}\big)$, where the gradient is clipped to $\|\bar{g}_{\text{sgd}}\|_2 \le C_{\text{sgd}}$ and the noise is $\xi_{\text{sgd}} \sim \mathcal{N}(0, \sigma^2 C_{\text{sgd}}^2 \mathbf{I})$. Because the DP noise is zero-mean and independent of the gradient, the cross term $\mathbb{E}[\bar{g}_{\text{sgd}}^\top \xi_{\text{sgd}}]$ vanishes. 
Substituting $\mathbb{E}[\|\xi_{\text{sgd}}\|_2^2] = d\sigma^2 C_{\text{sgd}}^2$ for the $d$-dimensional Gaussian noise, the expected squared Euclidean step size simplifies to
\begin{align}
    \mathbb{E}[\|\Delta \theta_{\text{sgd}}\|_2^2] 
    &= \eta_{\text{sgd}}^2 \Bigl( \|\bar{g}_{\text{sgd}}\|_2^2 + \frac{1}{B^2} \mathbb{E}[\|\xi_{\text{sgd}}\|_2^2] \Bigr) \nonumber \\
    &\le \eta_{\text{sgd}}^2 C_{\text{sgd}}^2 \Bigl(1 + \frac{d\sigma^2}{B^2}\Bigr).
\end{align}

For our DP-NGD framework, the parameter update is formulated as $\Delta \theta_{\text{ngd}} = -\eta_{\text{ngd}} F^{-1/2} \big(\bar{g}_{\text{ngd}} + \frac{1}{B} \xi_{\text{ngd}}\big)$. Taking the expectation and applying the Rayleigh quotient bound for symmetric matrices ($v^\top F^{-1} v \le \frac{1}{\lambda_{\text{min}}} \|v\|_2^2$), we obtain the upper bound
\begin{align}
    \mathbb{E}[\|\Delta \theta_{\text{ngd}}\|_2^2] 
    &= \eta_{\text{ngd}}^2 \Bigl( \bar{g}_{\text{ngd}}^\top F^{-1} \bar{g}_{\text{ngd}} + \frac{1}{B^2} \mathbb{E}\bigl[ \xi_{\text{ngd}}^\top F^{-1} \xi_{\text{ngd}} \bigr] \Bigr) \nonumber \\
    &\le \frac{\eta_{\text{ngd}}^2}{\lambda_{\text{min}}} C_{\text{ngd}}^2 \Bigl(1 + \frac{d\sigma^2}{B^2}\Bigr).
\end{align}

Assuming matched privacy configurations (identical batch size $B$ and noise multiplier $\sigma$), enforcing $\mathbb{E}[\|\Delta \theta_{\text{ngd}}\|_2^2] \le \mathbb{E}[\|\Delta \theta_{\text{sgd}}\|_2^2]$ allows the common factor $(1 + d\sigma^2/B^2)$ to cancel out. Solving for the minimum eigenvalue yields the safety bound $\lambda_{\text{safe}}$:
\begin{equation}
\label{equation:lambda_safe}
   \lambda_{\text{min}} \ge \left( \frac{\eta_{\text{ngd}} C_{\text{ngd}}}{\eta_{\text{sgd}} C_{\text{sgd}}} \right)^2 \triangleq \lambda_{\text{safe}}.
\end{equation}

\noindent This inequality reveals that as long as every eigenvalue satisfies $\lambda_i \ge \lambda_{\text{safe}}$, the expected Euclidean norm of the DP-NGD update is bounded by that of the DP-SGD baseline. To enforce this guarantee throughout training, each time the curvature matrix $F$ is estimated, we dynamically clamp its eigenvalues to this safety floor:
\begin{equation}
   \tilde{\lambda}_i = \max(\lambda_i, \lambda_{\text{safe}}).
\end{equation}
Preconditioning the gradient with this clamped curvature ensures optimization stability in the Euclidean space. Unlike standard NGD, which routinely adds a damping term $(F + \lambda I)^{-1}$ that indiscriminately perturbs all eigen-directions and sacrifices curvature accuracy in steep directions, our clamping operation is highly selective. It leaves well-conditioned, steep directions ($\lambda_i \ge \lambda_{\text{safe}}$) completely untouched, fully preserving their second-order acceleration, while safely lifting dangerously flat directions ($\lambda_i < \lambda_{\text{safe}}$) to the safety floor.

\paragraph{The Dynamic Clamping Schedule.}
Enforcing $\lambda_{\text{safe}}$ as a fixed clamping floor effectively suppresses parameter explosion, but is unnecessarily conservative during the high-SNR mid-training phase. The theoretical bound assumes a worst-case scenario where all eigenvalues are clamped to $\lambda_{\text{safe}}$. In practice, however, most eigen-directions remain steep ($\lambda_i \gg \lambda_{\text{safe}}$) during this mid-phase, and the actual Euclidean displacement is far below the safety limit. A static floor ignores this available margin, needlessly sacrificing curvature accuracy and suppressing second-order acceleration when Euclidean safety is already guaranteed.

To recover this lost efficiency during mid-training without compromising robust convergence in late-stage flat basins, we introduce a dynamic clamping schedule over a total of $T$ iterations:
\begin{equation}
\label{equation:dynamic_schedule}
\lambda_t = 
\begin{cases} 
\lambda_{\text{safe}} - (\lambda_{\text{safe}} - \lambda_{\text{base}}) \cdot \frac{t}{T_1}, & \text{if } 0 \le t < T_1 \\[4pt]
\lambda_{\text{base}} + (\lambda_{\text{safe}} - \lambda_{\text{base}}) \cdot \left( \frac{t - T_1}{T - T_1} \right)^p, & \text{if } T_1 \le t \le T 
\end{cases}
\end{equation}
where $T_1$ is a brief warmup period, $\lambda_{\text{base}} < \lambda_{\text{safe}}$ is a standard numerical stabilizer adopted from classical NGD, and $p > 1$ controls the steepness of the polynomial transition. Both $T_1$ (as a fraction of $T$) and $p$ are insensitive within practical ranges; a detailed sensitivity analysis is provided in Section~\ref{sec:hyper-parameter-robustness}.

This formulation naturally yields three distinct phases. In the initial warmup ($0 \le t < T_1$), curvature estimates are highly unreliable due to random initialization; thus, the floor starts at the conservative bound $\lambda_{\text{safe}}$ and linearly decays to $\lambda_{\text{base}}$ as training stabilizes. During the prolonged mid-phase ($T_1 \le t \ll T$), the exponent $p > 1$ creates an implicit acceleration plateau where $\lambda_t$ remains nearly flat around $\lambda_{\text{base}}$, fully unleashing second-order acceleration while the signal-to-noise ratio remains high. Finally, as the model enters flat basins in the late stage and the clamped fraction of eigenvalues inevitably rises, the polynomial tail drives $\lambda_t$ rapidly back to $\lambda_{\text{safe}}$ at $T$, seamlessly restoring the theoretical Euclidean safety guarantee.

\begin{algorithm}[tb]
\caption{Differentially Private Natural Gradient Descent (DP-NGD)}
\label{alg:dp_ngd}
\begin{algorithmic}[1]
\REQUIRE Private dataset $\mathcal{D}$ (size $N$), public dataset $\mathcal{D}_{\text{pub}}$, learning rate $\eta$, DP clipping bound $C$, noise multiplier $\sigma$, expected batch size $B$, curvature update interval $T_{\text{cur}}$, total iterations $T$.
\ENSURE Differentially private model parameters $\theta_T$.
\STATE Initialize model parameters $\theta_0$.
\STATE Compute theoretical safety floor $\lambda_{\text{safe}}$ via Eq.~(\ref{equation:lambda_safe}).
\FOR{$t = 0$ \TO $T-1$}
    
    \vspace{4pt}
    \STATE \textbf{[Module 1] Privacy-Free Curvature Estimation}
    \IF{$t \bmod T_{\text{cur}} = 0$}
        \STATE Compute empirical curvature $F$ on $\mathcal{D}_{\text{pub}}$ via K-FAC.
        \STATE Perform eigendecomposition: $F = Q \Lambda Q^\top$. \COMMENT{Cache $Q$, $\Lambda$ for $T_{\text{cur}}$ steps}
    \ENDIF

    \vspace{4pt}
    \STATE \textbf{[Module 2] Dynamic Curvature Clamping}
    \STATE Compute dynamic clamping floor $\lambda_t$ for step $t$ via Eq.~(\ref{equation:dynamic_schedule}).
    \STATE Clamp eigenvalues element-wise: $\tilde{\Lambda} = \max(\Lambda, \lambda_t \mathbf{I})$.
    \STATE Update whitening matrix: $\tilde{F}^{-1/2} = Q \tilde{\Lambda}^{-1/2} Q^\top$.

    \vspace{4pt}
    \STATE \textbf{[Module 3] Whitened-Space Update Mechanism}
    \STATE Sample batch $\mathcal{B}_t \subseteq \mathcal{D}$ with probability $q = B/N$. 
    \STATE Compute per-sample gradients: $g_i^{(t)} = \nabla_\theta \mathcal{L}(x_i, \theta_t)$ for each sample $x_i \in \mathcal{B}_t$.
    \STATE $\Delta\theta_t \leftarrow \text{WhitenedSpaceUpdate}\big(\{g_i^{(t)}\}_{x_i \in \mathcal{B}_t}, \tilde{F}^{-1/2}, C, \sigma\big)$. \COMMENT{Algorithm~\ref{alg:whitened_update}: Project $\to$ Clip $\to$ Noise $\to$ Inverse project}
    
    \vspace{4pt}
    \STATE Update parameters: $\theta_{t+1} = \theta_t + \Delta\theta_t$.
\ENDFOR
\STATE \textbf{Return} $\theta_T$.
\end{algorithmic}
\end{algorithm}

\subsection{The Complete Algorithm}

We consolidate our three core components into a unified, end-to-end DP-NGD framework. Algorithm~\ref{alg:dp_ngd} outlines the complete training procedure, which is structurally divided into three modules corresponding to the fundamental challenges.

\textbf{Module 1: Privacy-Free Curvature Estimation.} To avoid incurring significant privacy cost on curvature estimation, we precompute the empirical Fisher matrix $F$ exclusively on a small public auxiliary dataset $\mathcal{D}_{\text{pub}}$ using the K-FAC approximation. To amortize the computational overhead, curvature estimation and its subsequent eigendecomposition ($F = Q \Lambda Q^\top$) are performed once every $T_{\text{cur}}$ iterations and cached for reuse.

\textbf{Module 2: Dynamic Curvature Clamping.} Before constructing the whitening matrix $\tilde{F}^{-1/2}$, we clamp the curvature eigenvalues to a dynamically scheduled floor $\lambda_t$: $\tilde{\Lambda} = \max(\Lambda, \lambda_t \mathbf{I})$. Here, $\lambda_t$ is generated solely based on the current training step $t$ according to Eq.~(\ref{equation:dynamic_schedule}). This stage-aware clamping serves a threefold purpose: it stabilizes randomly initialized parameters during the early phase, fully unleashes second-order acceleration in the high-SNR mid-phase, and ensures robust convergence in late-stage flat basins.

\textbf{Module 3: Whitened-Space Update Mechanism.} A minibatch $\mathcal{B}_t$ is sampled from the private dataset via Poisson subsampling with probability $q$. We then compute the per-sample gradients $\{g_i^{(t)}\}_{x_i \in \mathcal{B}_t}$ and execute the whitened-space mechanism defined in Algorithm~\ref{alg:whitened_update}. This subroutine follows a seamless pipeline: projecting gradients into the whitened space, applying isotropic $L_2$-norm clipping, injecting isotropic Gaussian noise, and finally mapping the aggregated update back to the parameter space via $\tilde{F}^{-1/2}$. By the mathematical equivalence established in Theorem~\ref{theorem: curvature_alignment}, this mechanism preserves the anisotropic acceleration of NGD under rigorous DP guaranties.

\vspace{5pt}
\noindent \textbf{Remark: Computational Efficiency.} 
Despite the additional steps introduced by curvature estimation, the computational overhead of DP-NGD is well-amortized. This efficiency stems from three structural properties: (1) curvature estimation operates on a small public batch and is lazily updated only once every $T_{\text{cur}}$ iterations; (2) the K-FAC approximation reduces the eigendecomposition complexity from $\mathcal{O}(d^3)$ to $\mathcal{O}(d_{\text{in}}^3 + d_{\text{out}}^3)$, where $d_{\text{in}}, d_{\text{out}} \ll d$; and (3) second-order optimization substantially reduces the total number of iterations required to achieve baseline accuracy, effectively offsetting the additional per-step cost. A detailed empirical analysis of end-to-end wall-clock runtime is provided in Section~\ref{sec:wall-clock-time}.

\subsection{Privacy Analysis}
\label{sec:privacy analysis}

We now establish the privacy guarantee of the complete DP-NGD framework. Among the three modules of Algorithm~\ref{alg:dp_ngd}, Modules 1 and 2 incur zero privacy cost. The curvature matrix $F$ is estimated exclusively on the public auxiliary dataset $\mathcal{D}_{\text{pub}}$, which is disjoint from the private training set and lies strictly outside the adjacency relation that defines differential privacy. The dynamic clamping floor $\lambda_t$ is computed solely from the current training step $t$ and the predefined hyperparameters $\eta$ and $C$, without accessing any private data. Both modules are therefore privacy-free by construction.

Module 3, the whitened-space update, is the only component that accesses private data. It consists of four privacy-relevant operations: projecting gradients through the precomputed linear map $\tilde{F}^{-1/2}$; clipping them to a constant $L_2$-sensitivity $C$; injecting isotropic Gaussian noise parameterized by $\sigma$ and $C$; and performing the inverse projection via $\tilde{F}^{-1/2}$. Because the whitening matrix $\tilde{F}^{-1/2}$ is derived from public data and remains fixed for the current step, it acts as a linear transformation strictly independent of the private data. By the post-processing property of DP, both the projection and the inverse projection via $\tilde{F}^{-1/2}$ incur no additional privacy cost. Consequently, the privacy loss of Module 3 stems exclusively from the intermediate clipping and noise injection steps.

The clipping and noise injection steps in DP-NGD are executed identically to those in standard DP-SGD, while the only differences are the clipping space (whitened vs.\ Euclidean) and the clipping threshold ($C_\text{ngd}$ vs.\ $C_{\text{sgd}}$), as summarized in Table~\ref{table:privacy_cost_comparison}. The clipping space depends solely on the whitening matrix $\tilde{F}^{-1/2}$, which has been shown to leave the privacy budget unchanged. The clipping threshold acts symmetrically as a scaling factor for both the global sensitivity and the injected noise variance, mathematically canceling out in the privacy loss derivation. Thus, the step-wise privacy loss of DP-NGD is identical to that of standard DP-SGD. Given this strict equivalence, the end-to-end privacy guarantee over $T$ iterations composes identically to that of standard DP-SGD, and the final $(\epsilon, \delta)$-DP bound can be tightly tracked by any standard privacy accountant (e.g., Rényi DP or the PRV accountant) without modification.

\begin{table}[htbp]
\centering
\caption{Privacy cost comparison between DP-SGD and DP-NGD.The only differences are the clipping space and the clipping threshold. The former depends on a fixed linear mapping $\tilde{F}^{-1/2}$ that does not alter the privacy cost. The latter identically scales both global sensitivity and noise variance, canceling out in the privacy loss derivation. Therefore, identical per-step privacy loss is guaranteed, yielding the exact same total privacy budget.}
\label{table:privacy_cost_comparison}
\begin{tabular}{@{}lcc@{}}
\toprule
\textbf{Component} & \textbf{Standard DP-SGD} & \textbf{DP-NGD (Ours)} \\ 
\midrule
Clipping space & Euclidean ($\|\cdot\|_2$) & Whitened ($\|\cdot\|_{\tilde{F}^{-1}}$) \\ 
Clipping threshold & $C_{\text{sgd}}$ & $C_{\text{ngd}}$ \\ 
Noise mechanism & $\mathcal{N}\big(0, \sigma^2 C_{\text{sgd}}^2 \mathbf{I}\big)$ & $\mathcal{N}\big(0, \sigma^2 C_{\text{ngd}}^2 \mathbf{I}\big)$ \\ 
Per-step Privacy Loss & Determined by $\sigma, q$ & Determined by $\sigma, q$ \\ 
Total Budget & $(\epsilon,\delta)$ & $(\epsilon,\delta)$ \\ 
\bottomrule
\end{tabular}
\end{table}

\section{Evaluation}
\label{sec:experiments}

\subsection{Experimental Setup}
\label{subsec:setup}

\noindent \textbf{Datasets and Models.} We evaluate on three standard benchmarks: CIFAR-10~\cite{cifar10}, SVHN~\cite{svhn}, and UTKFace~\cite{utkface}, covering diverse image domains. We adopt WRN-16-4~\cite{dp_randp_wrn-16-4} and ResNet-20~\cite{resnet20} as our model architectures, both widely used in DP research. All Batch Normalization layers are replaced with Group Normalization (GN)~\cite{groupnorm,Groupnorm222} to prevent privacy leakage through batch statistics. A small public auxiliary set containing 500 images ($\approx 1\%$ of the training set size) is constructed by random sampling from out-of-distribution (OOD) sources—CIFAR-100 and ImageNet. All images are resized to $32 \times 32$ pixels.

\vspace{3pt}
\noindent \textbf{Baselines.}
The essence of NGD lies in utilizing curvature to \linebreak precondition (or project) the first-order gradients. Consequently, the advantage of our method hinges on the quality of the curvature estimation. To validate this, we construct a strictly controlled suite of baselines, which also leverage public data to construct their own preconditioners or projections. To ensure fair comparison, \linebreak every auxiliary-data-assisted baseline adopts exactly the same public dataset $\mathcal{D}_{\text{pub}}$.

\noindent $\bullet$ \textit{DP-SGD \& DP-SGD-PT}. DP-SGD is the canonical isotropic first-order baseline, implemented following the state-of-the-art recipe of DeepMind~\cite{deepmind}. To verify that our gains stem from geometric curvature extraction rather than mere feature initialization, we additionally introduce DP-SGD-PT, which is pre-trained on $\mathcal{D}_{\text{pub}}$ before DP training.

\noindent $\bullet$ \textit{AdaDPS}. AdaDPS~\cite{adadps} is a representative adaptive DP optimizer that uses $\mathcal{D}_{\text{pub}}$ to estimate gradient second moments, approximating the diagonal of the Fisher matrix $F$ as per parameter learning‑rate scalers. This baseline tests whether capturing only diagonal variance is sufficient under DP; the performance gap relative to DP-NGD reveals the necessity of capturing parameter cross-correlations through our block-diagonal preconditioner.

\noindent $\bullet$ \textit{GEP}. Gradient Embedding Perturbation (GEP)~\cite{GEP} uses $\mathcal{D}_{\text{pub}}$ to construct a low-rank subspace via PCA~\cite{pca111,pca222}, and subsequently projects private gradients onto this subspace to reduce the noise dimension. Comparing against GEP demonstrates that low-rank geometric projections tend to discard critical tail features, whereas our block-diagonal preconditioner operates in the full-dimensional parameter space, preserving the complete manifold topology without information truncation.

\vspace{3pt}
\noindent \textbf{Unified Training Configuration.}
To establish strong baselines, all methods share a unified modern DP training recipe: $16\times$ augmentation multiplicity~\cite{data_augmentation111}, exponential moving average (EMA)~\cite{ema111,ema222} for evaluation, and a uniformly large batch size of $B=4096$. We evaluate the privacy-utility trade-off across privacy budgets $\epsilon \in \{1.0, 2.0, 4.0, 6.0, 8.0\}$ with a fixed $\delta = 10^{-5}$. The expected batch size $B=4096$ implies the same subsampling rate $q = B/N$ for all methods. To guarantee absolute comparison fairness, the per-step noise multiplier $\sigma$ is independently calibrated for each method via the PRV accountant~\cite{prv_account1,prv_account2} provided by Opacus, ensuring that the target privacy budget is exactly exhausted at the final iteration.

Crucially, we emphasize that the varying choices of the clipping threshold $C$ do not compromise comparison fairness. Since $C$ symmetrically scales both the global sensitivity and the injected noise variance, it cancels out mathematically in the privacy loss derivation, thus affecting solely the optimization dynamics rather than the privacy budget. Notably, by the KL-DP duality established in Section~\ref{sec:whitened space update}, our clipping threshold simultaneously defines the KL trust-region radius of NGD. We therefore adopt a fixed $C = 10.0$, directly inherited from the standard NGD KL constraint. For the baseline methods, we set the clipping threshold $C$ to $1.0$ for DP-SGD and $2.0$ for AdaDPS, following their respective optimal configurations. For GEP, we adopt its recommended settings of $C=10.0$ for the gradient embedding and $C=2.0$ for the residual gradient.

\vspace{3pt}
\noindent \textbf{Hyperparameter Tuning.}
To report every method at its peak performance, we conduct a rigorous hyperparameter tuning protocol. For the universal optimization parameters---specifically the learning rate $\eta$ and total training epochs $T$---we perform a comprehensive grid search to identify the optimal utility-privacy trade-off for each individual method. For method-specific structural parameters (e.g., adaptive momentum factors), we strictly adopt the recommended values from the original papers; when explicit recommendations are absent, we search within a narrow range around the original experimental settings. This protocol ensures that any observed performance gaps reflect fundamental algorithmic advantages rather than suboptimal hyperparameter choices. For our DP-NGD, we set the auxiliary public dataset size to $|\mathcal{D}_{\text{pub}}| = 500$ (approximately $1\%$ of the private training data), the curvature update interval to $T_{\text{cur}} = 8$, and configure the dynamic clamping schedule with $T_1 = 0.1\,T$ and $p = 10$. The robustness of these choices is validated by the sensitivity analysis in Section~\ref{sec:hyper-parameter-robustness}.

\subsection{End-to-End Utility Evaluation}
\label{subsec:utility}

\begin{table*}[ht]
\footnotesize
\caption{\textbf{Test accuracy (\%) across three datasets and five privacy budgets. Results report the mean and standard deviation over 5 independent runs. Best results are in bold. All public-data-assisted methods use exactly the same auxiliary dataset.}}
\label{table:utility_evaluation}
\begin{center}
\renewcommand{\arraystretch}{1.3} 
\begin{tabular}{
    m{2.2cm}<{\centering} 
    m{2.8cm}<{\centering} 
    *{5}{c} 
}
\toprule 
\multirow{2}{*}{\textbf{Dataset (Model)}} &
\multirow{2}{*}{\textbf{Method}} &
{\textbf{$\epsilon = 1.0$}} &
{\textbf{$\epsilon = 2.0$}} &
{\textbf{$\epsilon = 4.0$}} &
{\textbf{$\epsilon = 6.0$}} &
{\textbf{$\epsilon = 8.0$}} \\
\cmidrule(lr){3-3} \cmidrule(lr){4-4} \cmidrule(lr){5-5} \cmidrule(lr){6-6} \cmidrule(lr){7-7}
& & {\textbf{ACC}} & {\textbf{ACC}} & {\textbf{ACC}} & {\textbf{ACC}} & {\textbf{ACC}} \\
\midrule 

{CIFAR-10} & {DP-SGD} & {58.16\% \textcolor{gray}{\scriptsize $\pm$0.35}} & {65.71\% \textcolor{gray}{\scriptsize $\pm$0.41}} & {72.73\% \textcolor{gray}{\scriptsize $\pm$0.48}} & {77.49\% \textcolor{gray}{\scriptsize $\pm$0.65}} & {80.42\% \textcolor{gray}{\scriptsize $\pm$0.52}} \\
{(WRN-16-4)} & {DP-SGD-PT} & {58.68\% \textcolor{gray}{\scriptsize $\pm$0.32}} & {65.83\% \textcolor{gray}{\scriptsize $\pm$0.49}} & {73.06\% \textcolor{gray}{\scriptsize $\pm$0.35}} & {77.94\% \textcolor{gray}{\scriptsize $\pm$0.61}} & {80.83\% \textcolor{gray}{\scriptsize $\pm$0.39}} \\
{} & {GEP} & {58.75\% \textcolor{gray}{\scriptsize $\pm$0.41}} & {66.48\% \textcolor{gray}{\scriptsize $\pm$0.36}} & {73.46\% \textcolor{gray}{\scriptsize $\pm$0.30}} & {78.24\% \textcolor{gray}{\scriptsize $\pm$0.58}} & {80.55\% \textcolor{gray}{\scriptsize $\pm$0.25}} \\
{} & {AdaDPS} & {59.25\% \textcolor{gray}{\scriptsize $\pm$0.45}} & {66.94\% \textcolor{gray}{\scriptsize $\pm$0.39}} & {73.59\% \textcolor{gray}{\scriptsize $\pm$0.54}} & {78.57\% \textcolor{gray}{\scriptsize $\pm$0.49}} & {80.62\% \textcolor{gray}{\scriptsize $\pm$0.66}} \\
{} & {\textbf{DP-NGD (Ours)}} & \textbf{60.78\% \textcolor{gray}{\scriptsize $\pm$0.38}} & \textbf{67.75\% \textcolor{gray}{\scriptsize $\pm$0.56}} & \textbf{74.02\% \textcolor{gray}{\scriptsize $\pm$0.44}} & \textbf{78.61\% \textcolor{gray}{\scriptsize $\pm$0.49}} & \textbf{81.09\% \textcolor{gray}{\scriptsize $\pm$0.45}} \\
\midrule

{SVHN} & {DP-SGD} & {79.34\% \textcolor{gray}{\scriptsize $\pm$0.58}} & {83.96\% \textcolor{gray}{\scriptsize $\pm$0.35}} & {87.01\% \textcolor{gray}{\scriptsize $\pm$0.41}} & {87.49\% \textcolor{gray}{\scriptsize $\pm$0.28}} & {88.60\% \textcolor{gray}{\scriptsize $\pm$0.55}} \\
{(WRN-16-4)} & {DP-SGD-PT} & {79.54\% \textcolor{gray}{\scriptsize $\pm$0.45}} & {84.11\% \textcolor{gray}{\scriptsize $\pm$0.52}} & {87.37\% \textcolor{gray}{\scriptsize $\pm$0.39}} & {88.59\% \textcolor{gray}{\scriptsize $\pm$0.46}} & {89.18\% \textcolor{gray}{\scriptsize $\pm$0.54}} \\
{} & {GEP} & {79.73\% \textcolor{gray}{\scriptsize $\pm$0.31}} & {86.21\% \textcolor{gray}{\scriptsize $\pm$0.37}} & {88.27\% \textcolor{gray}{\scriptsize $\pm$0.54}} & {88.84\% \textcolor{gray}{\scriptsize $\pm$0.60}} & {89.29\% \textcolor{gray}{\scriptsize $\pm$0.58}} \\
{} & {AdaDPS} & {79.87\% \textcolor{gray}{\scriptsize $\pm$0.35}} & {86.12\% \textcolor{gray}{\scriptsize $\pm$0.51}} & {88.23\% \textcolor{gray}{\scriptsize $\pm$0.38}} & {89.35\% \textcolor{gray}{\scriptsize $\pm$0.44}} & {89.98\% \textcolor{gray}{\scriptsize $\pm$0.41}} \\
{} & {\textbf{DP-NGD (Ours)}} & \textbf{84.65\% \textcolor{gray}{\scriptsize $\pm$0.37}} & \textbf{88.42\% \textcolor{gray}{\scriptsize $\pm$0.45}} & \textbf{90.70\% \textcolor{gray}{\scriptsize $\pm$0.52}} & \textbf{92.28\% \textcolor{gray}{\scriptsize $\pm$0.48}} & \textbf{92.67\% \textcolor{gray}{\scriptsize $\pm$0.47}} \\
\midrule

{UTKFace} & {DP-SGD} & {57.53\% \textcolor{gray}{\scriptsize $\pm$0.45}} & {65.02\% \textcolor{gray}{\scriptsize $\pm$0.40}} & {70.15\% \textcolor{gray}{\scriptsize $\pm$0.35}} & {73.10\% \textcolor{gray}{\scriptsize $\pm$0.31}} & {74.17\% \textcolor{gray}{\scriptsize $\pm$0.58}} \\
{(ResNet-20)} & {DP-SGD-PT} & {58.20\% \textcolor{gray}{\scriptsize $\pm$0.41}} & {65.46\% \textcolor{gray}{\scriptsize $\pm$0.36}} & {70.32\% \textcolor{gray}{\scriptsize $\pm$0.32}} & {73.29\% \textcolor{gray}{\scriptsize $\pm$0.58}} & {74.55\% \textcolor{gray}{\scriptsize $\pm$0.55}} \\
{} & {GEP} & {58.49\% \textcolor{gray}{\scriptsize $\pm$0.52}} & {67.90\% \textcolor{gray}{\scriptsize $\pm$0.45}} & {71.24\% \textcolor{gray}{\scriptsize $\pm$0.38}} & {74.43\% \textcolor{gray}{\scriptsize $\pm$0.44}} & {75.38\% \textcolor{gray}{\scriptsize $\pm$0.53}} \\
{} & {AdaDPS} & {58.75\% \textcolor{gray}{\scriptsize $\pm$0.55}} & {68.49\% \textcolor{gray}{\scriptsize $\pm$0.48}} & {71.19\% \textcolor{gray}{\scriptsize $\pm$0.41}} & {74.56\% \textcolor{gray}{\scriptsize $\pm$0.37}} & {75.58\% \textcolor{gray}{\scriptsize $\pm$0.33}} \\
{} & {\textbf{DP-NGD (Ours)}} & \textbf{62.95\% \textcolor{gray}{\scriptsize $\pm$0.55}} & \textbf{70.01\% \textcolor{gray}{\scriptsize $\pm$0.38}} & \textbf{73.88\% \textcolor{gray}{\scriptsize $\pm$0.48}} & \textbf{76.05\% \textcolor{gray}{\scriptsize $\pm$0.45}} & \textbf{76.22\% \textcolor{gray}{\scriptsize $\pm$0.52}} \\

\bottomrule 
\end{tabular}
\end{center}
\end{table*}

Table~\ref{table:utility_evaluation} summarizes the test accuracies across three standard benchmarks and five privacy budgets. DP-NGD consistently achieves the highest performance in every setting. The advantage is more pronounced under strict privacy constraints: at $\epsilon=1.0$, DP-NGD outperforms the most competitive baseline by $4.78\%$ on SVHN and $4.20\%$ on UTKFace. This aligns with our core insight---when first-order gradients are heavily corrupted by DP noise, coarse-grained curvature priors deliver the strongest acceleration benefit. As $\epsilon$ increases and the signal-to-noise ratio improves, the relative gains from second-order acceleration gradually moderate (e.g., sub-$1\%$ improvement on CIFAR-10 at $\epsilon=8.0$). This is because, when the noise magnitude is sufficiently low, the residual inaccuracies in curvature estimation from public data become a limiting factor for leveraging precise fine-grained gradient components.

Comparing DP-NGD against AdaDPS and GEP validates the necessity of our block-diagonal preconditioner. AdaDPS captures only diagonal variance, while GEP projects gradients into a low-rank subspace; both methods simplify the geometry and inevitably discard information---cross-parameter correlations in AdaDPS, and fine-grained tail features in GEP. By employing a block-diagonal preconditioner that operates in the full-dimensional parameter space, DP-NGD consistently surpasses both. The performance of DP-SGD-PT further sharpens this picture: simple pre-training on $\mathcal{D}_{\text{pub}}$ yields only marginal utility gains over standard DP-SGD and is largely outperformed by methods that actively exploit the geometry of the gradient space. This empirical gap confirms that the primary value of auxiliary public data lies not in mere feature initialization, but in extracting transferable geometric priors that reshape the optimization dynamics. Ultimately, these comparisons establish that the structural fidelity of the geometric prior directly dictates the utility gains under differential privacy.

\subsection{Convergence Efficiency and Speedup}

\begin{figure*}[t]
\centering
\includegraphics[width=0.85\textwidth]{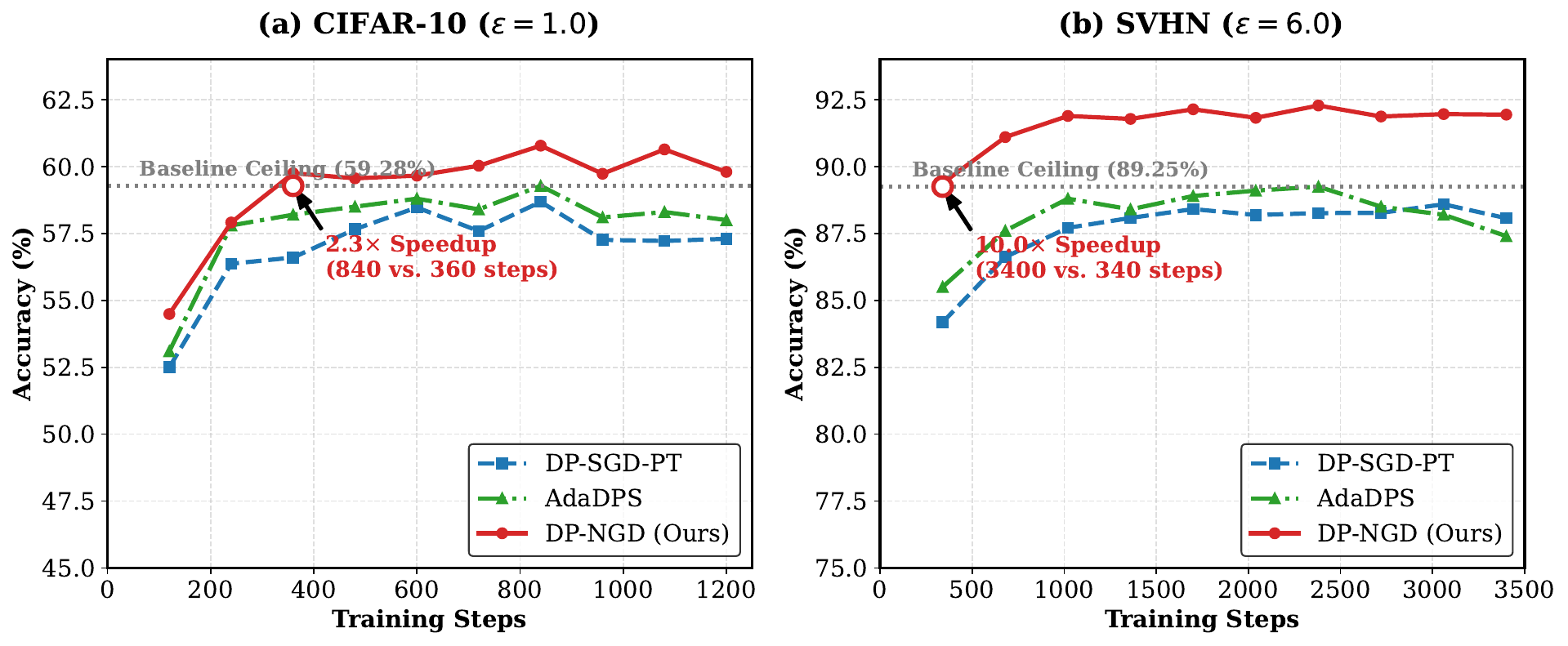} 
\vspace{-0.2cm}
\caption{Convergence efficiency vs.\ total training steps. We compare the test accuracy of DP-NGD, DP-SGD-PT, and AdaDPS under a strict privacy budget (CIFAR-10, $\epsilon=1.0$) and a moderate privacy budget (SVHN, $\epsilon=6.0$). Each point represents an independent full training run with a given total number of steps, where the per-step noise multiplier $\sigma$ is calibrated to exactly exhaust the target privacy budget at the final step. DP-NGD consistently achieves higher accuracy at every step budget, demonstrating superior per-step optimization efficiency.}
\label{fig:convergence_efficiency}
\end{figure*}

Figure~\ref{fig:convergence_efficiency} reports the test accuracy of DP-NGD, DP-SGD-PT, and AdaDPS at varying total training steps under a low privacy budget (CIFAR-10, $\epsilon=1.0$) and a moderate privacy budget (SVHN, $\epsilon=6.0$). Each point on the horizontal axis represents an independent full training run with a given total number of steps, where the per-step noise multiplier $\sigma$ is calibrated to exactly exhaust the target privacy budget at the final step. DP-NGD consistently achieves higher accuracy than the baselines when evaluated at the same training step. Moreover, it reaches the peak performance of the baselines with substantially fewer steps. On CIFAR-10, DP-NGD attains $59.75\%$ at step~360, already exceeding the best accuracy of AdaDPS ($58.28\%$ at step~840)---a $2.3\times$ reduction in the steps required to reach that utility level. On SVHN, DP-NGD reaches $89.38\%$ at step~340, surpassing the peak of both DP-SGD-PT and AdaDPS ($89.25\%$ at step~3400), corresponding to a $10.0\times$ efficiency gain. This advantage stems from the higher per-step optimization efficiency of curvature-aware updates. By preconditioning gradients with the block-diagonal inverse curvature, DP-NGD aligns each update with the dominant geometric directions of the loss landscape. This effectively suppresses wasteful oscillations in ill-conditioned terrain and extracting more meaningful signal per noisy step compared to diagonal (AdaDPS) or isotropic (DP-SGD-PT) approaches.

\subsection{Sample Efficiency of Public Data}
\label{sec:sample efficiency}

\begin{figure}[t] 
\centering
\includegraphics[width=0.95\columnwidth]{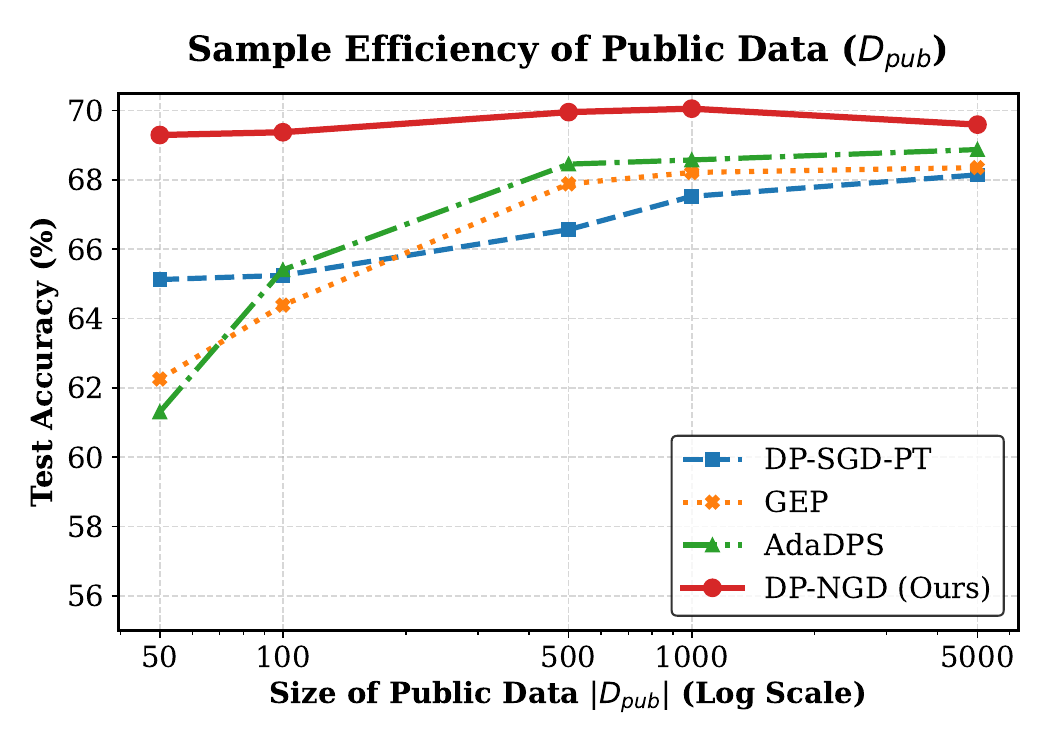} 
\vspace{-0.2cm} 
\caption{Sample efficiency of public auxiliary data. Test accuracy is reported on the UTKFace dataset ($\epsilon=2.0$) across varying public dataset sizes $|\mathcal{D}_{\text{pub}}|$. DP-NGD achieves near-peak accuracy with very few public samples and remains insensitive to $|\mathcal{D}_{\text{pub}}|$ across the entire range. In contrast, DP-SGD-PT, AdaDPS, and GEP all exhibit a clear dependence on the amount of public data, with their accuracy improving steadily with more data.}
\label{fig:sample_efficiency}
\vspace{-0.4cm} 
\end{figure}

Figure~\ref{fig:sample_efficiency} evaluates the impact of the public dataset size $|\mathcal{D}_{\text{pub}}|$ on model utility using the UTKFace dataset at $\epsilon=2.0$. By varying the number of public samples from $50$ to $5000$, we assess how effectively each method extracts geometric priors from limited auxiliary data. DP-NGD exhibits remarkable sample efficiency and a distinct insensitivity to $|\mathcal{D}_{\text{pub}}|$. With as few as $50$ public samples, our method already achieves $69.29\%$ test accuracy, and this performance fluctuates by less than $1\%$ across the entire evaluated range. In striking contrast, the baselines show a clear dependence on the volume of public data. While the accuracy of DP-SGD-PT, AdaDPS, and GEP improves steadily as $|\mathcal{D}_{\text{pub}}|$ increases, they consistently underperform DP-NGD even when utilizing the maximum available $5000$ samples. 

This stark difference directly reinforces our core claim: coarse-grained curvature priors are sufficient for effective second-order acceleration under DP. The K-FAC approximation compresses curvature degrees of freedom from $\mathcal{O}(d^2)$ to layer-width dimensions, enabling reliable estimation of the dominant geometric directions from very few public samples. Although additional public data might refine fine-grained curvature components, these precise directions are ultimately overwhelmed by DP noise during training, resulting in diminishing returns. Conversely, baselines relying on diagonal scaling (AdaDPS) or low-rank subspaces (GEP) require significantly more samples to construct statistically stable surrogates of the loss geometry. Furthermore, DP-SGD-PT, which exploits public data merely for weight initialization, yields the weakest per-sample benefit. Overall, these results confirm that extracting structural geometric priors is a fundamentally more data-efficient strategy to leverage auxiliary information.

\subsection{Ablation Studies}

\begin{table}[ht]
\footnotesize
\caption{Necessity of curvature decoupling on CIFAR-10 at $\epsilon=2.0$ and $4.0$. We evaluate the impact of allocating a fraction of the privacy budget ($20\%$ or $50\%$) to private curvature estimation. Even a modest budget split severely degrades accuracy, dropping performance below the first-order DP-SGD-PT baseline that uses no curvature at all.} 
\label{tab:ablation_decoupling}
\begin{center}
\renewcommand{\arraystretch}{1.2}
\begin{tabular}{l c c}
\toprule
\textbf{Variant Configuration} & \textbf{$\epsilon=2.0$} & \textbf{$\epsilon=4.0$} \\
\midrule
DP-SGD-FT (No Curvature) & 65.86\% & 72.75\% \\
\midrule
Private DP-NGD ($\epsilon$ split 50/50) & 54.35\% & 61.23\% \\
Private DP-NGD ($\epsilon$ split 20/80) & 56.12\% & 64.50\% \\
\midrule
\textbf{Public DP-NGD (Decoupled)} & \textbf{67.78\%} & \textbf{73.88\%} \\
\bottomrule
\end{tabular}
\end{center}
\vspace{-0.4cm}
\end{table}

\noindent \textbf{The Necessity of Curvature Decoupling.} We verify the necessity of decoupling curvature estimation from private data by comparing three strategies on CIFAR‑10 at $\epsilon=2.0$ and $4.0$: (i) our default DP‑NGD, which estimates the curvature matrix entirely on public data and preserves the full privacy budget for gradient updates; (ii) a variant that estimates curvature on private data, allocating a fraction of the privacy budget to the curvature computation and the remainder to gradient updates; and (iii) DP‑SGD‑PT, which merely pre‑trains on public data without any curvature estimation. When $20\%$ of the privacy budget is diverted to private curvature estimation, the accuracy drops dramatically from $67.78\%$ to $56.12\%$ at $\epsilon=2.0$, and from $73.88\%$ to $64.50\%$ at $\epsilon=4.0$. A $50\%$ split worsens the results further, yielding only $54.35\%$ and $61.23\%$, respectively. In sharp contrast, our public‑data‑driven curvature estimation incurs zero privacy cost and consistently outperforms even the DP‑SGD‑PT baseline ($65.86\%$ and $72.75\%$). These results expose a fundamental tension: under a fixed privacy budget, splitting that budget between curvature and gradients severely degrades the signal‑to‑noise ratio of the updates that actually drive learning, leading to catastrophic utility collapse. Decoupling curvature extraction from private data is therefore not merely a convenience but a necessary condition for making second‑order optimization viable under differential privacy.

\begin{table}[ht]
\footnotesize
\caption{Ablation on whitened-space clipping on CIFAR-10 ($\epsilon \in \{2.0, 4.0\}$). We compare our whitened-space DP operations against two alternatives: pre-conditioning clipping (clipping raw gradients before applying curvature) and post-conditioning clipping (clipping the natural gradient directly). Pre-conditioning suffers from geometric incompatibility, while post-conditioning distorts the anisotropic descent direction despite being privacy-compliant. Whitened-space clipping uniquely satisfies both the strict DP sensitivity bound and the directional integrity of second-order updates.}
\label{tab:ablation_clipping}
\begin{center}
\renewcommand{\arraystretch}{1.2}
\begin{tabular}{l c c}
\toprule
\textbf{Clipping Strategy} & \textbf{$\epsilon=2.0$} & \textbf{$\epsilon=4.0$} \\
\midrule
Pre-conditioning Clipping & 61.25\% & 66.82\% \\
Post-conditioning Clipping & 64.54\% & 68.48\% \\
\midrule
\textbf{Whitened-Space Clipping (Ours)} & \textbf{67.78\%} & \textbf{73.88\%} \\
\bottomrule
\end{tabular}
\end{center}
\vspace{-0.4cm}
\end{table}

\noindent \textbf{Ablation on Whitened-Space Clipping.} 
We validate the necessity of whitened-space clipping on CIFAR-10 ($\epsilon \in \{2.0, 4.0\}$) by comparing it against two alternatives: clipping raw gradients before applying curvature (\textit{pre-conditioning}), and clipping the resulting natural gradients directly (\textit{post-conditioning}). Our method achieves the highest accuracy ($67.78\%$ and $73.88\%$), substantially outperforming both pre-conditioning ($61.25\%$ and $66.82\%$) and post-conditioning ($64.54\%$ and $68.48\%$). Pre-conditioning performs worst: isotropic $L_2$ clipping in the original space treats all gradient directions uniformly, irreversibly corrupting the gradient structure before curvature information can be incorporated. Post-conditioning, although privacy-compliant because the curvature matrix comes from public data, applies an isotropic clip to the already anisotropic natural gradient, thereby distorting its carefully aligned descent direction. Whitened-space clipping resolves this tension by performing DP operations in the $F^{-1/2}$-whitened space, where an isotropic $L_2$ constraint is mathematically equivalent to the anisotropic Fisher-norm constraint required by NGD. This preserves both the strict DP sensitivity bound and the directional integrity of second-order updates.

\begin{figure}[ht]
\centering
\includegraphics[width=0.9\columnwidth]{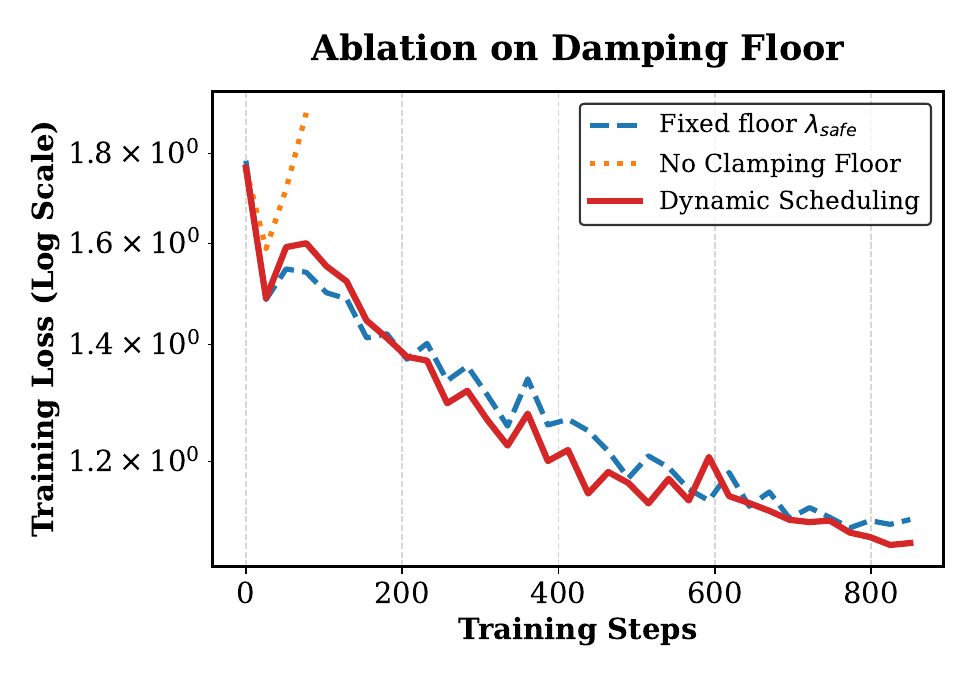}
\vspace{-0.2cm}
\caption{Ablation on the clamping floor using UTKFace ($\epsilon=2.0$). Training loss trajectories for three strategies are shown. Training diverges rapidly without clamping, while both clamped strategies converge stably. Our dynamic schedule further relaxes the floor to avoid over-suppressing well-conditioned eigenvalues, yielding lower loss in high-SNR mid-phase and higher final accuracy.}
\label{fig:ablation_damping}
\vspace{-0.4cm}
\end{figure}

\noindent \textbf{Ablation on the Clamping Floor.} 
We evaluate the impact of the dynamic clamping schedule on training stability and final utility using UTKFace ($\epsilon=2.0$). We compare three configurations: (i) our proposed dynamic schedule, (ii) a fixed clamping floor set to the theoretical safety bound $\lambda_{\text{safe}}$, and (iii) no clamping floor at all. Training without any clamping diverges rapidly within the first few steps, producing NaN values and confirming that bounding only the $L_2$-norm in the whitened space is insufficient to guarantee Euclidean stability of the resulting natural gradient. Both the fixed $\lambda_{\text{safe}}$ and our dynamic schedule converge stably. In the early phase ($< 200$ steps), when curvature estimates are unreliable, the fixed $\lambda_{\text{safe}}$ yields a slightly lower training loss. However, as training stabilizes and enters the high-SNR mid-phase, the fixed floor becomes overly conservative. The loss curve of our dynamic schedule consistently lies below that of the fixed $\lambda_{\text{safe}}$ from the mid-phase onward, indicating that a static floor persistently lifts eigenvalues that are already well-conditioned, needlessly sacrificing curvature accuracy and suppressing the second-order acceleration. Our dynamic schedule avoids this by relaxing the floor toward a low base value, achieving consistently lower training loss and a markedly higher final test accuracy ($70.06\%$ vs.\ $68.53\%$).

\subsection{End-to-End System Overhead}
\label{sec:wall-clock-time}

\begin{table}[ht]
\vspace{0.2cm}
\footnotesize
\setlength{\tabcolsep}{4pt} 
\caption{End-to-end system overhead comparison between DP-NGD and DP-SGD-PT. 
Despite a higher per-step cost, DP-NGD achieves comparable or lower total runtime due to sharply fewer steps. All runs use a single NVIDIA RTX 3090.}
\label{tab:system_overhead}
\begin{center}
\renewcommand{\arraystretch}{1.2}
\begin{tabular}{@{} c c c c c @{}}
\toprule
\textbf{\makecell[t]{Dataset\\($\epsilon$)}} & 
\textbf{\makecell[t]{Method}} & 
\textbf{\makecell[t]{Per-step Time\\\textmd{(s)}}} & 
\textbf{\makecell[t]{Total Steps}} & 
\textbf{\makecell[t]{Total Time\\\textmd{(Hrs)}}} \\
\midrule
\multirow{2}{*}{\makecell{CIFAR-10\\($\epsilon=1.0$)}} 
& DP-SGD-PT & 14.2 & 840 & 3.31 \\
& \textbf{DP-NGD} & \textbf{38.2} & \textbf{360} & \textbf{3.82} \\
\midrule
\multirow{2}{*}{\makecell{SVHN\\($\epsilon=6.0$)}} 
& DP-SGD-PT & 13.6 & 3400 & 12.84 \\
& \textbf{DP-NGD} & \textbf{39.8} & \textbf{340} & \textbf{3.76} \\
\bottomrule
\end{tabular}
\end{center}
\vspace{-0.4cm}
\end{table}

Table~\ref{tab:system_overhead} reports the per-step wall-clock time, the number of training steps required to reach the peak accuracy of DP-SGD-PT, and the resulting total wall-clock time for both methods on two representative benchmarks. All measurements are performed on a single NVIDIA GeForce RTX 3090 GPU under identical batch size, augmentation, and privacy configurations. DP-NGD incurs a higher per-step cost than DP-SGD-PT ($38$--$40$\,s vs.\ $14$\,s), primarily due to the whitened-space projection and the periodic curvature estimation. However, this overhead is substantially offset by the sharply reduced number of steps resulting from curvature-aware optimization. On CIFAR-10 ($\epsilon=1.0$), DP-NGD reaches the peak accuracy of DP-SGD-PT in only $360$ steps, compared to $840$ steps for the baseline. The total runtime is $3.82$ hours, comparable to that of DP-SGD-PT ($3.31$ hours), while achieving higher final accuracy. On SVHN ($\epsilon=6.0$), the advantage is dramatic: DP-NGD requires only $340$ steps to match the baseline accuracy that demands $3{,}400$ steps for DP-SGD-PT. The total runtime drops from $12.84$ hours to $3.76$ hours---a $3.4\times$ speedup. This stark contrast shows that in ill-conditioned optimization landscapes, the ability of curvature-aware preconditioning to suppress wasteful oscillations far outweighs its per-step computational cost. 
Overall, these results confirm that DP-NGD offers practical 
accelerator for differentially private training.

\subsection{Hyperparameter Robustness}
\label{sec:hyper-parameter-robustness}

\begin{table}[ht]
\footnotesize
\setlength{\tabcolsep}{6pt}
\caption{Hyperparameter sensitivity on UTKFace ($\epsilon=2.0$). We evaluate the robustness of DP-NGD to the curvature update interval $T_{\text{cur}}$, the polynomial exponent $p$, and the transition step ratio $T_1$. Test accuracy remains stable across all three parameters, with fluctuations within a narrow range. Default values are underlined.}
\label{tab:sensitivity}
\begin{center}

\begin{tabular}{@{} c c c c c c @{}}
\toprule
\multicolumn{6}{c}{\textbf{(a) Sensitivity to Curvature Update Interval} ($T_{cur}$)} \\
\midrule
Value & 1 & 4 & \underline{8} & 16 & 32 \\
Accuracy & 70.14\% & 70.02\% & 70.05\% & 69.87\% & 67.53\% \\
\bottomrule
\end{tabular}

\vspace{0.3cm} 

\begin{tabular}{@{} c c c c c c @{}}
\toprule
\multicolumn{6}{c}{\textbf{(b) Sensitivity to Polynomial Exponent} ($p$)} \\
\midrule
Value & 2 & 3 & 5 & \underline{10} & 15 \\
Accuracy & 68.96\% & 69.21\% & 69.08\% & 70.05\% & 69.63\% \\
\bottomrule
\end{tabular}

\vspace{0.3cm} 

\begin{tabular}{@{} c c c c c c @{}}
\toprule
\multicolumn{6}{c}{\textbf{(c) Sensitivity to Transition Step} ($T_1$ ratio)} \\
\midrule
Value & 0.05 & \underline{0.1} & 0.15 & 0.2 & 0.25 \\
Accuracy & 69.51\% & 70.05\% & 69.25\% & 69.46\% & 69.26\% \\
\bottomrule
\end{tabular}

\end{center}
\vspace{-0.4cm}
\end{table}

Table~\ref{tab:sensitivity} evaluates the sensitivity of DP-NGD to its three key hyperparameters on UTKFace ($\epsilon=2.0$): the curvature update interval $T_{\text{cur}}$, the polynomial exponent $p$, and the transition step ratio $T_1$. Across all three parameters, the test accuracy remains stable within a narrow range, confirming the robustness of our framework to these design choices. The curvature update interval $T_{\text{cur}}$ can be extended from $1$ to $16$ steps with less than a $0.3\%$ accuracy drop; only at an extremely sparse update frequency of $32$ steps does accuracy decline noticeably to $67.53\%$, still a competitive result. The polynomial exponent $p$ and the transition ratio $T_1$ exhibit similar insensitivity: varying $p$ from $2$ to $15$ or $T_1$ from $0.05$ to $0.25$ induces fluctuations of less than $1\%$. This robustness reflects the underlying stability of coarse-grained geometric priors: once the dominant eigen-directions of the curvature are captured from public data, the precise details of the dynamic clamping schedule exert only a marginal influence on optimization dynamics. The high tolerance to these hyperparameters significantly reduces the tuning burden in practical deployments.

\section{Conclusion}
\label{sec:conclusion}

We proposed DP-NGD, a practical framework that enables effective second-order optimization under differential privacy. By decoupling curvature estimation from private data, reconciling isotropic DP operations with anisotropic natural gradient updates through a whitened-space mechanism, and dynamically clamping curvature eigenvalues, DP-NGD systematically resolves the three fundamental obstacles that have so far prevented the integration of NGD and DP. Extensive experiments demonstrate that DP-NGD consistently achieves state-of-the-art accuracy across multiple benchmarks and privacy budgets, breaking through the utility ceilings of first-order baselines with up to $10\times$ convergence speedup. Notably, these gains are obtained with a trivially small public auxiliary dataset and without incurring additional privacy cost. Our work opens up several promising directions, including extending coarse-grained curvature priors to large-scale models, relaxing the reliance on public data through fully private geometric estimation, and exploring curvature-aware optimization under other privacy notions.

\appendix

\section{Derivation of the Whitening Matrix $W = F^{-1/2}$}
\label{app:whitening_derivation}

To reconcile this conflict between DP and NGD, we exploit two critical property: (1) the DP privacy guarantee holds strictly as long as the sensitivity is bounded under the $L_2$ norm in \textit{any} coordinate space; (2) The acceleration of NGD requires only the final parameter updates to align with the local loss geometry. This insight motivates a natural strategy: transfer all DP operations into an intermediate projected space where the isotropic $L_2$ constraint can be safely applied, and then map the result back to the original parameter space so that the final update recovers the curvature-aligned geometry of NGD.

We now formalize the two geometric requirements. First, NGD constrains the per-step distributional shift via the KL divergence. Under the standard second-order Taylor expansion and the NGD update rule $\Delta\theta = -\eta F^{-1} g$, the KL constraint reduces to a bound on the \textit{Fisher norm} of the gradient:
\begin{equation}\label{eq:fisher_norm}
    \|g\|_{F^{-1}}^2 = g^\top F^{-1} g.
\end{equation}

Second, DP requires that the update vector used for clipping and noising has a bounded $L_2$ norm. In a projected space defined by a linear map $W$, the $L_2$ norm of the projected gradient is
\begin{equation}\label{eq:projected_norm}
    \|W g\|_2^2 = g^\top (W^\top W) g.
\end{equation}

To guarantee that the isotropic DP operations in the projected space do not corrupt the curvature-aware information, the $L_2$ geometry of the projected space must be intrinsically equivalent to the KL geometry of the original space. Equating the quadratic forms in Eqs.~\eqref{eq:fisher_norm} and~\eqref{eq:projected_norm} yields the fundamental condition
\begin{equation}\label{eq:condition}
    W^\top W = F^{-1}.
\end{equation}
This equation demands a matrix $W$ whose symmetric product recovers the inverse Fisher. Since the Fisher matrix $F$ and its inverse $F^{-1}$ are symmetric and positive semi-definite, the unique symmetric positive semi-definite solution to Eq.~\eqref{eq:condition} is the inverse square root of $F$:
\begin{equation}\label{eq:solution}
    W = F^{-1/2}.
\end{equation}

The derivation reveals why the $F^{-1/2}$-whitened space is the unique geometry that aligns DP and NGD. In this space, bounding the $L_2$ norm of the projected gradient $\|F^{-1/2} g\|_2$ is mathematically identical to bounding the Fisher norm $\|g\|_{F^{-1}}$, which in turn is equivalent to respecting the NGD trust region. Consequently, performing standard isotropic DP clipping and noising in the $F^{-1/2}$-whitened space, and then mapping back via $F^{1/2}$, automatically yields an anisotropic, curvature-aware update in the original parameter space. This provides the theoretical foundation for the KL-DP duality.

\paragraph{Efficient computation via K-FAC.}
In deep networks, constructing the full $F^{-1/2}$ would be prohibitive. The Kronecker-Factored Approximate Curvature (K-FAC) resolves this by factorizing each layer's Fisher block as $F_l \approx A_l \otimes G_l$. The whitening matrix can then be efficiently computed via eigendecomposition of the small Kronecker factors:
\[
    F_l^{-1/2} \approx (Q_A \Lambda_A^{-1/2} Q_A^\top) \otimes (Q_G \Lambda_G^{-1/2} Q_G^\top),
\]
reducing the complexity from $\mathcal{O}(d^3)$ to $\mathcal{O}(d_{\text{in}}^3 + d_{\text{out}}^3)$, where $d_{\text{in}}, d_{\text{out}} \ll d$. This makes the whitened-space update computationally tractable.

\section{Optimal Hyperparameter Configurations}
\label{app:optimal_configs}

Table~\ref{tab:optimal_configs} lists the best hyperparameter configurations identified via grid search for DP-NGD across all datasets and privacy budgets reported in the main text. These settings correspond to the results presented in Table~\ref{table:utility_evaluation}.

\begin{table}[ht]
\centering
\caption{Optimal learning rates and total training steps for DP-NGD. All configurations use a fixed batch size of $B=4096$, curvature update interval $T_{\text{cur}}=8$, warmup fraction $T_1=0.1$, polynomial exponent $p=10$, and auxiliary public dataset size $|\mathcal{D}_{\text{pub}}|=500$.}
\label{tab:optimal_configs}
\begin{tabular}{@{}l c c c@{}}
\toprule
\textbf{Dataset (Model)} & \textbf{$\epsilon$} & \textbf{Total Steps $T$} & \textbf{Learning Rate $\eta$} \\
\midrule
\multirow{5}{*}{CIFAR-10 (WRN-16-4)} 
& 1.0 & 840 & 0.01 \\
& 2.0 & 1320 & 0.01 \\
& 4.0 & 1920 & 0.01 \\
& 6.0 & 1800 & 0.02 \\
& 8.0 & 2640 & 0.02 \\
\midrule
\multirow{5}{*}{SVHN (WRN-16-4)} 
& 1.0 & 1360 & 0.01 \\
& 2.0 & 1870 & 0.01 \\
& 4.0 & 2040 & 0.01 \\
& 6.0 & 2380 & 0.02 \\
& 8.0 & 3400 & 0.02 \\
\midrule
\multirow{5}{*}{UTKFace (ResNet-20)} 
& 1.0 & 550 & 0.01 \\
& 2.0 & 850 & 0.01 \\
& 4.0 & 900 & 0.02 \\
& 6.0 & 1100 & 0.02 \\
& 8.0 & 1700 & 0.02 \\
\bottomrule
\end{tabular}
\end{table}

\clearpage

\bibliographystyle{ACM-Reference-Format}
\bibliography{sample}


\begin{thebibliography}{45}


\ifx \showCODEN    \undefined \def \showCODEN     #1{\unskip}     \fi
\ifx \showDOI      \undefined \def \showDOI       #1{#1}\fi
\ifx \showISBNx    \undefined \def \showISBNx     #1{\unskip}     \fi
\ifx \showISBNxiii \undefined \def \showISBNxiii  #1{\unskip}     \fi
\ifx \showISSN     \undefined \def \showISSN      #1{\unskip}     \fi
\ifx \showLCCN     \undefined \def \showLCCN      #1{\unskip}     \fi
\ifx \shownote     \undefined \def \shownote      #1{#1}          \fi
\ifx \showarticletitle \undefined \def \showarticletitle #1{#1}   \fi
\ifx \showURL      \undefined \def \showURL       {\relax}        \fi
\providecommand\bibfield[2]{#2}
\providecommand\bibinfo[2]{#2}
\providecommand\natexlab[1]{#1}
\providecommand\showeprint[2][]{arXiv:#2}

\bibitem[\protect\citeauthoryear{Abadi, Chu, Goodfellow, McMahan, Mironov, Talwar, and Zhang}{Abadi et~al\mbox{.}}{2016}]%
        {dpsgd1}
\bibfield{author}{\bibinfo{person}{Mart{\'i}n Abadi}, \bibinfo{person}{Andy Chu}, \bibinfo{person}{Ian~J. Goodfellow}, \bibinfo{person}{H.~B. McMahan}, \bibinfo{person}{Ilya Mironov}, \bibinfo{person}{Kunal Talwar}, {and} \bibinfo{person}{Li Zhang}.} \bibinfo{year}{2016}\natexlab{}.
\newblock \showarticletitle{Deep Learning with Differential Privacy}.
\newblock \bibinfo{journal}{\emph{Proceedings of the 2016 ACM SIGSAC Conference on Computer and Communications Security}} (\bibinfo{year}{2016}).
\newblock
\urldef\tempurl%
\url{https://api.semanticscholar.org/CorpusID:207241585}
\showURL{%
\tempurl}


\bibitem[\protect\citeauthoryear{Abouelnaga, Ali, Rady, and Moustafa}{Abouelnaga et~al\mbox{.}}{2016}]%
        {cifar10}
\bibfield{author}{\bibinfo{person}{Yehya Abouelnaga}, \bibinfo{person}{Ola~S Ali}, \bibinfo{person}{Hager Rady}, {and} \bibinfo{person}{Mohamed Moustafa}.} \bibinfo{year}{2016}\natexlab{}.
\newblock \showarticletitle{Cifar-10: Knn-based ensemble of classifiers}. In \bibinfo{booktitle}{\emph{2016 International Conference on Computational Science and Computational Intelligence (CSCI)}}. IEEE, \bibinfo{pages}{1192--1195}.
\newblock


\bibitem[\protect\citeauthoryear{Ahn and Cutkosky}{Ahn and Cutkosky}{2024}]%
        {ema222}
\bibfield{author}{\bibinfo{person}{Kwangjun Ahn} {and} \bibinfo{person}{Ashok Cutkosky}.} \bibinfo{year}{2024}\natexlab{}.
\newblock \showarticletitle{Adam with model exponential moving average is effective for nonconvex optimization}.
\newblock \bibinfo{journal}{\emph{ArXiv}}  \bibinfo{volume}{abs/2405.18199} (\bibinfo{year}{2024}).
\newblock
\urldef\tempurl%
\url{https://api.semanticscholar.org/CorpusID:270068108}
\showURL{%
\tempurl}


\bibitem[\protect\citeauthoryear{Babu, Vinothini, Solaimalai, Vanitha, Yadav, Sukania, Vijayan, and Srinivasan}{Babu et~al\mbox{.}}{2024}]%
        {Groupnorm222}
\bibfield{author}{\bibinfo{person}{Bachina~Harish Babu}, \bibinfo{person}{V.~R. Vinothini}, \bibinfo{person}{Gautam Solaimalai}, \bibinfo{person}{R. Vanitha}, \bibinfo{person}{Ajay~Singh Yadav}, \bibinfo{person}{P. Sukania}, \bibinfo{person}{Vishal Vijayan}, {and} \bibinfo{person}{Revanth Srinivasan}.} \bibinfo{year}{2024}\natexlab{}.
\newblock \showarticletitle{Theoretical optimization of group size in group normalization for enhanced deep neural network training}.
\newblock \bibinfo{journal}{\emph{AIP Conference Proceedings}} (\bibinfo{year}{2024}).
\newblock
\urldef\tempurl%
\url{https://api.semanticscholar.org/CorpusID:274008853}
\showURL{%
\tempurl}


\bibitem[\protect\citeauthoryear{Bao, Pittaluga, Kumar, and Bindschaedler}{Bao et~al\mbox{.}}{2023}]%
        {data_augmentation111}
\bibfield{author}{\bibinfo{person}{Wenxuan Bao}, \bibinfo{person}{F. Pittaluga}, \bibinfo{person}{Vijay Kumar}, {and} \bibinfo{person}{Vincent Bindschaedler}.} \bibinfo{year}{2023}\natexlab{}.
\newblock \showarticletitle{DP-Mix: Mixup-based Data Augmentation for Differentially Private Learning}.
\newblock \bibinfo{journal}{\emph{ArXiv}}  \bibinfo{volume}{abs/2311.01295} (\bibinfo{year}{2023}).
\newblock
\urldef\tempurl%
\url{https://api.semanticscholar.org/CorpusID:264935065}
\showURL{%
\tempurl}


\bibitem[\protect\citeauthoryear{Block, Bun, Desai, Shetty, and Wu}{Block et~al\mbox{.}}{2024}]%
        {dp_public_data1}
\bibfield{author}{\bibinfo{person}{Adam Block}, \bibinfo{person}{Mark Bun}, \bibinfo{person}{Rathin Desai}, \bibinfo{person}{Abhishek Shetty}, {and} \bibinfo{person}{Zhiwei~Steven Wu}.} \bibinfo{year}{2024}\natexlab{}.
\newblock \showarticletitle{Oracle-Efficient Differentially Private Learning with Public Data}.
\newblock \bibinfo{journal}{\emph{ArXiv}}  \bibinfo{volume}{abs/2402.09483} (\bibinfo{year}{2024}).
\newblock
\urldef\tempurl%
\url{https://api.semanticscholar.org/CorpusID:267682343}
\showURL{%
\tempurl}


\bibitem[\protect\citeauthoryear{Boix-Adser{\`a}}{Boix-Adser{\`a}}{2025}]%
        {inductive_prior2}
\bibfield{author}{\bibinfo{person}{Enric Boix-Adser{\`a}}.} \bibinfo{year}{2025}\natexlab{}.
\newblock \showarticletitle{On the inductive bias of infinite-depth ResNets and the bottleneck rank}.
\newblock \bibinfo{journal}{\emph{ArXiv}}  \bibinfo{volume}{abs/2501.19149} (\bibinfo{year}{2025}).
\newblock
\urldef\tempurl%
\url{https://api.semanticscholar.org/CorpusID:276079304}
\showURL{%
\tempurl}


\bibitem[\protect\citeauthoryear{Bu, Wang, Zha, and Karypis}{Bu et~al\mbox{.}}{2022}]%
        {dpsgd2}
\bibfield{author}{\bibinfo{person}{Zhiqi Bu}, \bibinfo{person}{Yu-Xiang Wang}, \bibinfo{person}{Sheng Zha}, {and} \bibinfo{person}{George Karypis}.} \bibinfo{year}{2022}\natexlab{}.
\newblock \showarticletitle{Automatic Clipping: Differentially Private Deep Learning Made Easier and Stronger}.
\newblock \bibinfo{journal}{\emph{ArXiv}}  \bibinfo{volume}{abs/2206.07136} (\bibinfo{year}{2022}).
\newblock
\urldef\tempurl%
\url{https://api.semanticscholar.org/CorpusID:249674655}
\showURL{%
\tempurl}


\bibitem[\protect\citeauthoryear{Chandaliya, Kumar, Harjani, and Nain}{Chandaliya et~al\mbox{.}}{2019}]%
        {utkface}
\bibfield{author}{\bibinfo{person}{Praveen~Kumar Chandaliya}, \bibinfo{person}{V.Chandra Kumar}, \bibinfo{person}{Mayank Harjani}, {and} \bibinfo{person}{Neeta Nain}.} \bibinfo{year}{2019}\natexlab{}.
\newblock \showarticletitle{SCDAE: Ethnicity and Gender Alteration on CLF and UTKFace Dataset}. In \bibinfo{booktitle}{\emph{International Conference on Computer Vision and Image Processing}}.
\newblock
\urldef\tempurl%
\url{https://api.semanticscholar.org/CorpusID:214730428}
\showURL{%
\tempurl}


\bibitem[\protect\citeauthoryear{Chaudhari, Choromańska, Soatto, LeCun, Baldassi, Borgs, Chayes, Sagun, and Zecchina}{Chaudhari et~al\mbox{.}}{2016}]%
        {flat_direction1}
\bibfield{author}{\bibinfo{person}{Pratik Chaudhari}, \bibinfo{person}{Anna Choromańska}, \bibinfo{person}{Stefano Soatto}, \bibinfo{person}{Yann LeCun}, \bibinfo{person}{Carlo Baldassi}, \bibinfo{person}{Christian Borgs}, \bibinfo{person}{Jennifer~Tour Chayes}, \bibinfo{person}{Levent Sagun}, {and} \bibinfo{person}{Riccardo Zecchina}.} \bibinfo{year}{2016}\natexlab{}.
\newblock \showarticletitle{Entropy-SGD: biasing gradient descent into wide valleys}.
\newblock \bibinfo{journal}{\emph{Journal of Statistical Mechanics: Theory and Experiment}}  \bibinfo{volume}{2019} (\bibinfo{year}{2016}).
\newblock
\urldef\tempurl%
\url{https://api.semanticscholar.org/CorpusID:13807351}
\showURL{%
\tempurl}


\bibitem[\protect\citeauthoryear{Dangel, M{\"u}ller, and Zeinhofer}{Dangel et~al\mbox{.}}{2024}]%
        {kfac3}
\bibfield{author}{\bibinfo{person}{Felix Dangel}, \bibinfo{person}{Johannes M{\"u}ller}, {and} \bibinfo{person}{Marius Zeinhofer}.} \bibinfo{year}{2024}\natexlab{}.
\newblock \showarticletitle{Kronecker-factored approximate curvature for physics-informed neural networks}.
\newblock \bibinfo{journal}{\emph{Advances in Neural Information Processing Systems}}  \bibinfo{volume}{37} (\bibinfo{year}{2024}), \bibinfo{pages}{34582--34636}.
\newblock


\bibitem[\protect\citeauthoryear{De, Berrada, Hayes, Smith, and Balle}{De et~al\mbox{.}}{2022}]%
        {deepmind}
\bibfield{author}{\bibinfo{person}{Soham De}, \bibinfo{person}{Leonard Berrada}, \bibinfo{person}{Jamie Hayes}, \bibinfo{person}{Samuel~L. Smith}, {and} \bibinfo{person}{Borja Balle}.} \bibinfo{year}{2022}\natexlab{}.
\newblock \showarticletitle{Unlocking High-Accuracy Differentially Private Image Classification through Scale}.
\newblock \bibinfo{journal}{\emph{ArXiv}}  \bibinfo{volume}{abs/2204.13650} (\bibinfo{year}{2022}).
\newblock
\urldef\tempurl%
\url{https://api.semanticscholar.org/CorpusID:248427073}
\showURL{%
\tempurl}


\bibitem[\protect\citeauthoryear{Dong, Liang, and Yi}{Dong et~al\mbox{.}}{2022}]%
        {Dong_dp}
\bibfield{author}{\bibinfo{person}{Wei Dong}, \bibinfo{person}{Yuting Liang}, {and} \bibinfo{person}{Ke Yi}.} \bibinfo{year}{2022}\natexlab{}.
\newblock \showarticletitle{Differentially Private Covariance Revisited}.
\newblock \bibinfo{journal}{\emph{ArXiv}}  \bibinfo{volume}{abs/2205.14324} (\bibinfo{year}{2022}).
\newblock
\urldef\tempurl%
\url{https://api.semanticscholar.org/CorpusID:249191345}
\showURL{%
\tempurl}


\bibitem[\protect\citeauthoryear{Dwork}{Dwork}{2006}]%
        {dp1}
\bibfield{author}{\bibinfo{person}{Cynthia Dwork}.} \bibinfo{year}{2006}\natexlab{}.
\newblock \showarticletitle{Differential Privacy}. In \bibinfo{booktitle}{\emph{International Colloquium on Automata, Languages and Programming}}.
\newblock
\urldef\tempurl%
\url{https://api.semanticscholar.org/CorpusID:2565493}
\showURL{%
\tempurl}


\bibitem[\protect\citeauthoryear{Dwork and Roth}{Dwork and Roth}{2014}]%
        {dp2}
\bibfield{author}{\bibinfo{person}{Cynthia Dwork} {and} \bibinfo{person}{Aaron Roth}.} \bibinfo{year}{2014}\natexlab{}.
\newblock \showarticletitle{The Algorithmic Foundations of Differential Privacy}.
\newblock \bibinfo{journal}{\emph{Found. Trends Theor. Comput. Sci.}}  \bibinfo{volume}{9} (\bibinfo{year}{2014}), \bibinfo{pages}{211--407}.
\newblock
\urldef\tempurl%
\url{https://api.semanticscholar.org/CorpusID:207178262}
\showURL{%
\tempurl}


\bibitem[\protect\citeauthoryear{Feng, Mohammady, Hong, Yan, Kundu, Wang, and Hong}{Feng et~al\mbox{.}}{2024}]%
        {moments_account2}
\bibfield{author}{\bibinfo{person}{Shuya Feng}, \bibinfo{person}{Meisam Mohammady}, \bibinfo{person}{Hanbin Hong}, \bibinfo{person}{Shenao Yan}, \bibinfo{person}{Ashish Kundu}, \bibinfo{person}{Binghui Wang}, {and} \bibinfo{person}{Yuan Hong}.} \bibinfo{year}{2024}\natexlab{}.
\newblock \showarticletitle{Harmonizing Differential Privacy Mechanisms for Federated Learning: Boosting Accuracy and Convergence}.
\newblock \bibinfo{journal}{\emph{Proceedings of the Fifteenth ACM Conference on Data and Application Security and Privacy}} (\bibinfo{year}{2024}).
\newblock
\urldef\tempurl%
\url{https://api.semanticscholar.org/CorpusID:278238496}
\showURL{%
\tempurl}


\bibitem[\protect\citeauthoryear{Foret, Kleiner, Mobahi, and Neyshabur}{Foret et~al\mbox{.}}{2020}]%
        {flat_direction2}
\bibfield{author}{\bibinfo{person}{Pierre Foret}, \bibinfo{person}{Ariel Kleiner}, \bibinfo{person}{Hossein Mobahi}, {and} \bibinfo{person}{Behnam Neyshabur}.} \bibinfo{year}{2020}\natexlab{}.
\newblock \showarticletitle{Sharpness-Aware Minimization for Efficiently Improving Generalization}.
\newblock \bibinfo{journal}{\emph{ArXiv}}  \bibinfo{volume}{abs/2010.01412} (\bibinfo{year}{2020}).
\newblock
\urldef\tempurl%
\url{https://api.semanticscholar.org/CorpusID:222134093}
\showURL{%
\tempurl}


\bibitem[\protect\citeauthoryear{Gopi, Lee, and Wutschitz}{Gopi et~al\mbox{.}}{2021}]%
        {prv_account1}
\bibfield{author}{\bibinfo{person}{Sivakanth Gopi}, \bibinfo{person}{Yin~Tat Lee}, {and} \bibinfo{person}{Lukas Wutschitz}.} \bibinfo{year}{2021}\natexlab{}.
\newblock \showarticletitle{Numerical Composition of Differential Privacy}. In \bibinfo{booktitle}{\emph{Neural Information Processing Systems}}.
\newblock
\urldef\tempurl%
\url{https://api.semanticscholar.org/CorpusID:235358383}
\showURL{%
\tempurl}


\bibitem[\protect\citeauthoryear{Greenacre, Groenen, Hastie, D’Enza, Markos, and Tuzhilina}{Greenacre et~al\mbox{.}}{2003}]%
        {pca222}
\bibfield{author}{\bibinfo{person}{Michael Greenacre}, \bibinfo{person}{Patrick J.~F. Groenen}, \bibinfo{person}{Trevor~J. Hastie}, \bibinfo{person}{Alfonso~Iodice D’Enza}, \bibinfo{person}{Angelos~I. Markos}, {and} \bibinfo{person}{Elena Tuzhilina}.} \bibinfo{year}{2003}\natexlab{}.
\newblock \showarticletitle{Principal Component Analysis}.
\newblock \bibinfo{journal}{\emph{Technometrics}}  \bibinfo{volume}{45} (\bibinfo{year}{2003}), \bibinfo{pages}{276 -- 276}.
\newblock
\urldef\tempurl%
\url{https://api.semanticscholar.org/CorpusID:2534141}
\showURL{%
\tempurl}


\bibitem[\protect\citeauthoryear{Grosse and Martens}{Grosse and Martens}{2016}]%
        {kfac2}
\bibfield{author}{\bibinfo{person}{Roger Grosse} {and} \bibinfo{person}{James Martens}.} \bibinfo{year}{2016}\natexlab{}.
\newblock \showarticletitle{A kronecker-factored approximate fisher matrix for convolution layers}. In \bibinfo{booktitle}{\emph{International Conference on Machine Learning}}. PMLR, \bibinfo{pages}{573--582}.
\newblock


\bibitem[\protect\citeauthoryear{Guth and M'enard}{Guth and M'enard}{2024}]%
        {low_level_statistic1}
\bibfield{author}{\bibinfo{person}{Florentin Guth} {and} \bibinfo{person}{Brice M'enard}.} \bibinfo{year}{2024}\natexlab{}.
\newblock \showarticletitle{On the universality of neural encodings in CNNs}.
\newblock \bibinfo{journal}{\emph{ArXiv}}  \bibinfo{volume}{abs/2409.19460} (\bibinfo{year}{2024}).
\newblock
\urldef\tempurl%
\url{https://api.semanticscholar.org/CorpusID:272987466}
\showURL{%
\tempurl}


\bibitem[\protect\citeauthoryear{Huo, Fan, Li, Chen, Gao, and Li}{Huo et~al\mbox{.}}{2020}]%
        {resnet20}
\bibfield{author}{\bibinfo{person}{Tianjiao Huo}, \bibinfo{person}{Jiaqi Fan}, \bibinfo{person}{Xin Li}, \bibinfo{person}{Hong Chen}, \bibinfo{person}{Bingzhao Gao}, {and} \bibinfo{person}{Xuesong Li}.} \bibinfo{year}{2020}\natexlab{}.
\newblock \showarticletitle{Traffic Sign Recognition Based on ResNet-20 and Deep Mutual Learning}.
\newblock \bibinfo{journal}{\emph{2020 Chinese Automation Congress (CAC)}} (\bibinfo{year}{2020}), \bibinfo{pages}{4770--4774}.
\newblock
\urldef\tempurl%
\url{https://api.semanticscholar.org/CorpusID:231737679}
\showURL{%
\tempurl}


\bibitem[\protect\citeauthoryear{Jolliffe and Cadima}{Jolliffe and Cadima}{2016}]%
        {pca111}
\bibfield{author}{\bibinfo{person}{Ian~T. Jolliffe} {and} \bibinfo{person}{Jorge Cadima}.} \bibinfo{year}{2016}\natexlab{}.
\newblock \showarticletitle{Principal component analysis: a review and recent developments}.
\newblock \bibinfo{journal}{\emph{Philosophical Transactions of the Royal Society A: Mathematical, Physical and Engineering Sciences}}  \bibinfo{volume}{374} (\bibinfo{year}{2016}).
\newblock
\urldef\tempurl%
\url{https://api.semanticscholar.org/CorpusID:20101754}
\showURL{%
\tempurl}


\bibitem[\protect\citeauthoryear{Kudo and Tajima}{Kudo and Tajima}{2022}]%
        {fim1}
\bibfield{author}{\bibinfo{person}{Daigo Kudo} {and} \bibinfo{person}{Hiroyasu Tajima}.} \bibinfo{year}{2022}\natexlab{}.
\newblock \showarticletitle{Fisher information matrix as a resource measure in the resource theory of asymmetry with general connected-Lie-group symmetry}.
\newblock \bibinfo{journal}{\emph{Physical Review A}} (\bibinfo{year}{2022}).
\newblock
\urldef\tempurl%
\url{https://api.semanticscholar.org/CorpusID:248562504}
\showURL{%
\tempurl}


\bibitem[\protect\citeauthoryear{Kunstner, Hennig, and Balles}{Kunstner et~al\mbox{.}}{2019}]%
        {natural_gradient2}
\bibfield{author}{\bibinfo{person}{Frederik Kunstner}, \bibinfo{person}{Philipp Hennig}, {and} \bibinfo{person}{Lukas Balles}.} \bibinfo{year}{2019}\natexlab{}.
\newblock \showarticletitle{Limitations of the empirical fisher approximation for natural gradient descent}.
\newblock \bibinfo{journal}{\emph{Advances in neural information processing systems}}  \bibinfo{volume}{32} (\bibinfo{year}{2019}).
\newblock


\bibitem[\protect\citeauthoryear{Li, Zaheer, Reddi, and Smith}{Li et~al\mbox{.}}{2022}]%
        {adadps}
\bibfield{author}{\bibinfo{person}{Tian Li}, \bibinfo{person}{Manzil Zaheer}, \bibinfo{person}{Sashank~J. Reddi}, {and} \bibinfo{person}{Virginia Smith}.} \bibinfo{year}{2022}\natexlab{}.
\newblock \showarticletitle{Private Adaptive Optimization with Side Information}. In \bibinfo{booktitle}{\emph{International Conference on Machine Learning}}.
\newblock
\urldef\tempurl%
\url{https://api.semanticscholar.org/CorpusID:246823521}
\showURL{%
\tempurl}


\bibitem[\protect\citeauthoryear{Lowy, Li, Huang, and Razaviyayn}{Lowy et~al\mbox{.}}{2023}]%
        {dp_public_data2}
\bibfield{author}{\bibinfo{person}{Andrew Lowy}, \bibinfo{person}{Zeman Li}, \bibinfo{person}{Tianjian Huang}, {and} \bibinfo{person}{Meisam Razaviyayn}.} \bibinfo{year}{2023}\natexlab{}.
\newblock \showarticletitle{Optimal Differentially Private Learning with Public Data}.
\newblock \bibinfo{journal}{\emph{ArXiv}}  \bibinfo{volume}{abs/2306.15056} (\bibinfo{year}{2023}).
\newblock
\urldef\tempurl%
\url{https://api.semanticscholar.org/CorpusID:275750699}
\showURL{%
\tempurl}


\bibitem[\protect\citeauthoryear{Madhulika and Sampath}{Madhulika and Sampath}{2022}]%
        {svhn}
\bibfield{author}{\bibinfo{person}{PSS Madhulika} {and} \bibinfo{person}{Nalini Sampath}.} \bibinfo{year}{2022}\natexlab{}.
\newblock \showarticletitle{An application of normalizer free neural networks on the SVHN dataset}. In \bibinfo{booktitle}{\emph{2022 International conference on applied artificial intelligence and computing (ICAAIC)}}. IEEE, \bibinfo{pages}{238--242}.
\newblock


\bibitem[\protect\citeauthoryear{Martens}{Martens}{2020}]%
        {natural_gradient}
\bibfield{author}{\bibinfo{person}{James Martens}.} \bibinfo{year}{2020}\natexlab{}.
\newblock \showarticletitle{New insights and perspectives on the natural gradient method}.
\newblock \bibinfo{journal}{\emph{Journal of Machine Learning Research}} \bibinfo{volume}{21}, \bibinfo{number}{146} (\bibinfo{year}{2020}), \bibinfo{pages}{1--76}.
\newblock


\bibitem[\protect\citeauthoryear{Mehmeti-G{\"o}pel and Wand}{Mehmeti-G{\"o}pel and Wand}{2025}]%
        {inductive_prior1}
\bibfield{author}{\bibinfo{person}{Christian H. X.~Ali Mehmeti-G{\"o}pel} {and} \bibinfo{person}{Michael Wand}.} \bibinfo{year}{2025}\natexlab{}.
\newblock \showarticletitle{ResNets Are Deeper Than You Think}.
\newblock \bibinfo{journal}{\emph{ArXiv}}  \bibinfo{volume}{abs/2506.14386} (\bibinfo{year}{2025}).
\newblock
\urldef\tempurl%
\url{https://api.semanticscholar.org/CorpusID:279410014}
\showURL{%
\tempurl}


\bibitem[\protect\citeauthoryear{Meiser and Mohammadi}{Meiser and Mohammadi}{2018}]%
        {prv_account2}
\bibfield{author}{\bibinfo{person}{Sebastian Meiser} {and} \bibinfo{person}{Esfandiar Mohammadi}.} \bibinfo{year}{2018}\natexlab{}.
\newblock \showarticletitle{Tight on Budget?: Tight Bounds for r-Fold Approximate Differential Privacy}.
\newblock \bibinfo{journal}{\emph{Proceedings of the 2018 ACM SIGSAC Conference on Computer and Communications Security}} (\bibinfo{year}{2018}).
\newblock
\urldef\tempurl%
\url{https://api.semanticscholar.org/CorpusID:52504941}
\showURL{%
\tempurl}


\bibitem[\protect\citeauthoryear{Mironov}{Mironov}{2017}]%
        {rdp1}
\bibfield{author}{\bibinfo{person}{Ilya Mironov}.} \bibinfo{year}{2017}\natexlab{}.
\newblock \showarticletitle{R{\'e}nyi Differential Privacy}.
\newblock \bibinfo{journal}{\emph{2017 IEEE 30th Computer Security Foundations Symposium (CSF)}} (\bibinfo{year}{2017}), \bibinfo{pages}{263--275}.
\newblock
\urldef\tempurl%
\url{https://api.semanticscholar.org/CorpusID:9386213}
\showURL{%
\tempurl}


\bibitem[\protect\citeauthoryear{Mironov, Talwar, and Zhang}{Mironov et~al\mbox{.}}{2019}]%
        {rdp2}
\bibfield{author}{\bibinfo{person}{Ilya Mironov}, \bibinfo{person}{Kunal Talwar}, {and} \bibinfo{person}{Li Zhang}.} \bibinfo{year}{2019}\natexlab{}.
\newblock \showarticletitle{R{\'e}nyi Differential Privacy of the Sampled Gaussian Mechanism}.
\newblock \bibinfo{journal}{\emph{ArXiv}}  \bibinfo{volume}{abs/1908.10530} (\bibinfo{year}{2019}).
\newblock
\urldef\tempurl%
\url{https://api.semanticscholar.org/CorpusID:201660159}
\showURL{%
\tempurl}


\bibitem[\protect\citeauthoryear{Morales-Brotons, Vogels, and Hendrikx}{Morales-Brotons et~al\mbox{.}}{2024}]%
        {ema111}
\bibfield{author}{\bibinfo{person}{Daniel Morales-Brotons}, \bibinfo{person}{Thijs Vogels}, {and} \bibinfo{person}{Hadrien Hendrikx}.} \bibinfo{year}{2024}\natexlab{}.
\newblock \showarticletitle{Exponential Moving Average of Weights in Deep Learning: Dynamics and Benefits}.
\newblock \bibinfo{journal}{\emph{Trans. Mach. Learn. Res.}}  \bibinfo{volume}{2024} (\bibinfo{year}{2024}).
\newblock
\urldef\tempurl%
\url{https://api.semanticscholar.org/CorpusID:270355881}
\showURL{%
\tempurl}


\bibitem[\protect\citeauthoryear{Pajarinen, Thai, Akrour, Peters, and Neumann}{Pajarinen et~al\mbox{.}}{2019}]%
        {kl_trust_region2}
\bibfield{author}{\bibinfo{person}{Joni Pajarinen}, \bibinfo{person}{Hong~Linh Thai}, \bibinfo{person}{Riad Akrour}, \bibinfo{person}{Jan Peters}, {and} \bibinfo{person}{Gerhard Neumann}.} \bibinfo{year}{2019}\natexlab{}.
\newblock \showarticletitle{Compatible natural gradient policy search}.
\newblock \bibinfo{journal}{\emph{Machine Learning}}  \bibinfo{volume}{108} (\bibinfo{year}{2019}), \bibinfo{pages}{1443 -- 1466}.
\newblock
\urldef\tempurl%
\url{https://api.semanticscholar.org/CorpusID:59842948}
\showURL{%
\tempurl}


\bibitem[\protect\citeauthoryear{Rodriguez, Farr, Farr, and Mandel}{Rodriguez et~al\mbox{.}}{2013}]%
        {fim2}
\bibfield{author}{\bibinfo{person}{Carl~L Rodriguez}, \bibinfo{person}{Benjamin Farr}, \bibinfo{person}{Will~M Farr}, {and} \bibinfo{person}{Ilya Mandel}.} \bibinfo{year}{2013}\natexlab{}.
\newblock \showarticletitle{Inadequacies of the Fisher information matrix in gravitational-wave parameter estimation}.
\newblock \bibinfo{journal}{\emph{Physical Review D—Particles, Fields, Gravitation, and Cosmology}} \bibinfo{volume}{88}, \bibinfo{number}{8} (\bibinfo{year}{2013}), \bibinfo{pages}{084013}.
\newblock


\bibitem[\protect\citeauthoryear{Tang, Panda, Sehwag, and Mittal}{Tang et~al\mbox{.}}{2023}]%
        {dp_randp_wrn-16-4}
\bibfield{author}{\bibinfo{person}{Xinyu Tang}, \bibinfo{person}{Ashwinee Panda}, \bibinfo{person}{Vikash Sehwag}, {and} \bibinfo{person}{Prateek Mittal}.} \bibinfo{year}{2023}\natexlab{}.
\newblock \showarticletitle{Differentially Private Image Classification by Learning Priors from Random Processes}.
\newblock \bibinfo{journal}{\emph{J. Priv. Confidentiality}}  \bibinfo{volume}{15} (\bibinfo{year}{2023}).
\newblock
\urldef\tempurl%
\url{https://api.semanticscholar.org/CorpusID:259129508}
\showURL{%
\tempurl}


\bibitem[\protect\citeauthoryear{Venkatesan, Gattupalli, and Li}{Venkatesan et~al\mbox{.}}{2016}]%
        {low_level_statistic2}
\bibfield{author}{\bibinfo{person}{Ragav Venkatesan}, \bibinfo{person}{Vijetha Gattupalli}, {and} \bibinfo{person}{Baoxin Li}.} \bibinfo{year}{2016}\natexlab{}.
\newblock \showarticletitle{On the generality of neural image features}.
\newblock \bibinfo{journal}{\emph{2016 IEEE International Conference on Image Processing (ICIP)}} (\bibinfo{year}{2016}), \bibinfo{pages}{41--45}.
\newblock
\urldef\tempurl%
\url{https://api.semanticscholar.org/CorpusID:9264808}
\showURL{%
\tempurl}


\bibitem[\protect\citeauthoryear{Wang, yang Gao, Zhang, Shen, and Su}{Wang et~al\mbox{.}}{2022}]%
        {privacy_account1}
\bibfield{author}{\bibinfo{person}{Hua Wang}, \bibinfo{person}{Sheng yang Gao}, \bibinfo{person}{Huanyu Zhang}, \bibinfo{person}{Milan Shen}, {and} \bibinfo{person}{Weijie Su}.} \bibinfo{year}{2022}\natexlab{}.
\newblock \showarticletitle{Edgeworth Accountant: An Analytical Approach to Differential Privacy Composition}.
\newblock \bibinfo{journal}{\emph{J. Amer. Statist. Assoc.}} (\bibinfo{year}{2022}).
\newblock
\urldef\tempurl%
\url{https://api.semanticscholar.org/CorpusID:249538149}
\showURL{%
\tempurl}


\bibitem[\protect\citeauthoryear{Wang, Balle, and Kasiviswanathan}{Wang et~al\mbox{.}}{2018}]%
        {moments_account1}
\bibfield{author}{\bibinfo{person}{Yu-Xiang Wang}, \bibinfo{person}{Borja Balle}, {and} \bibinfo{person}{Shiva~Prasad Kasiviswanathan}.} \bibinfo{year}{2018}\natexlab{}.
\newblock \showarticletitle{Subsampled R{\'e}nyi Differential Privacy and Analytical Moments Accountant}. In \bibinfo{booktitle}{\emph{International Conference on Artificial Intelligence and Statistics}}.
\newblock
\urldef\tempurl%
\url{https://api.semanticscholar.org/CorpusID:51893629}
\showURL{%
\tempurl}


\bibitem[\protect\citeauthoryear{Wu and He}{Wu and He}{2018}]%
        {groupnorm}
\bibfield{author}{\bibinfo{person}{Yuxin Wu} {and} \bibinfo{person}{Kaiming He}.} \bibinfo{year}{2018}\natexlab{}.
\newblock \showarticletitle{Group normalization}. In \bibinfo{booktitle}{\emph{Proceedings of the European conference on computer vision (ECCV)}}. \bibinfo{pages}{3--19}.
\newblock


\bibitem[\protect\citeauthoryear{Xu, Xuan, Zhang, and Lu}{Xu et~al\mbox{.}}{2024}]%
        {kl_trust_region1}
\bibfield{author}{\bibinfo{person}{Haotian Xu}, \bibinfo{person}{Junyu Xuan}, \bibinfo{person}{Guangquan Zhang}, {and} \bibinfo{person}{Jie Lu}.} \bibinfo{year}{2024}\natexlab{}.
\newblock \showarticletitle{Trust region policy optimization via entropy regularization for Kullback-Leibler divergence constraint}.
\newblock \bibinfo{journal}{\emph{Neurocomputing}}  \bibinfo{volume}{589} (\bibinfo{year}{2024}), \bibinfo{pages}{127716}.
\newblock
\urldef\tempurl%
\url{https://api.semanticscholar.org/CorpusID:269195555}
\showURL{%
\tempurl}


\bibitem[\protect\citeauthoryear{Yu, Zhang, Chen, and Liu}{Yu et~al\mbox{.}}{2021}]%
        {GEP}
\bibfield{author}{\bibinfo{person}{Da Yu}, \bibinfo{person}{Huishuai Zhang}, \bibinfo{person}{Wei Chen}, {and} \bibinfo{person}{Tie-Yan Liu}.} \bibinfo{year}{2021}\natexlab{}.
\newblock \showarticletitle{Do Not Let Privacy Overbill Utility: Gradient Embedding Perturbation for Private Learning}.
\newblock \bibinfo{journal}{\emph{ArXiv}}  \bibinfo{volume}{abs/2102.12677} (\bibinfo{year}{2021}).
\newblock
\urldef\tempurl%
\url{https://api.semanticscholar.org/CorpusID:232046284}
\showURL{%
\tempurl}


\bibitem[\protect\citeauthoryear{Zhu, Fioretto, and Hentenryck}{Zhu et~al\mbox{.}}{2022}]%
        {post-processing1}
\bibfield{author}{\bibinfo{person}{Keyu Zhu}, \bibinfo{person}{Ferdinando Fioretto}, {and} \bibinfo{person}{Pascal~Van Hentenryck}.} \bibinfo{year}{2022}\natexlab{}.
\newblock \showarticletitle{Post-processing of Differentially Private Data: A Fairness Perspective}.
\newblock \bibinfo{journal}{\emph{ArXiv}}  \bibinfo{volume}{abs/2201.09425} (\bibinfo{year}{2022}).
\newblock
\urldef\tempurl%
\url{https://api.semanticscholar.org/CorpusID:246240328}
\showURL{%
\tempurl}


\bibitem[\protect\citeauthoryear{Zhu, Hentenryck, and Fioretto}{Zhu et~al\mbox{.}}{2020}]%
        {post-processing2}
\bibfield{author}{\bibinfo{person}{Keyu Zhu}, \bibinfo{person}{Pascal~Van Hentenryck}, {and} \bibinfo{person}{Ferdinando Fioretto}.} \bibinfo{year}{2020}\natexlab{}.
\newblock \showarticletitle{Bias and Variance of Post-processing in Differential Privacy}.
\newblock \bibinfo{journal}{\emph{ArXiv}}  \bibinfo{volume}{abs/2010.04327} (\bibinfo{year}{2020}).
\newblock
\urldef\tempurl%
\url{https://api.semanticscholar.org/CorpusID:222272017}
\showURL{%
\tempurl}


\end{thebibliography}

\end{document}